\documentclass{article}

\usepackage{microtype}
\usepackage{graphicx}
\usepackage{booktabs} % for professional tables
\usepackage{amsmath}
\usepackage{amsfonts}
\usepackage{amsthm}
\usepackage{mdframed}
\usepackage{algorithmic,algorithm}
\usepackage{url}
\usepackage{subfig}
\usepackage{pdfpages}

\usepackage{hyperref}

\usepackage[accepted]{icml2019}

\newcommand{\eg}{\emph{e.g.}}

\newcommand{\tildenabla}{\tilde{\nabla}}

 \newcommand{\Xcal}{\mathcal{X}}

\newcommand{\Ccal}{\mathcal{C}}

\newcommand{\vs}{\vspace*{-0.15cm}}
\newcommand{\vsmall}{\vspace*{-0.1cm}}

\renewcommand\epsilon\varepsilon
\renewcommand\ln\log

\renewcommand\star{*}
\def\defin{\triangleq}
\def\Real{{\mathbb{R}}}

\theoremstyle{plain}
\newtheorem{theorem}{Theorem}
\newtheorem{lemma}{Lemma}
\newtheorem{appendixlemma}{Lemma}[section]
\newtheorem{corollary}{Corollary}
\newtheorem{proposition}{Proposition}

\newcommand\custombox[1]{%
   \vspace*{0cm}
   \begin{mdframed}[leftmargin=0pt,innerleftmargin=6pt,innerrightmargin=6pt]
      \parindent=15pt#1
   \end{mdframed}
   \vspace*{0cm}
}

\def\Real{\mathbb R}
\def\E{\mathbb E}
\def\defin{:=}

\newcommand\kmone{{k\text{--}1}}

\newcommand\Fcal{{\mathcal F}}

\DeclareMathOperator*{\argmin}{argmin}
\DeclareMathOperator*{\Prox}{Prox}
\renewcommand\cite\citep

\icmltitlerunning{Estimate Sequences for Variance-Reduced Stochastic Composite Optimization}

\begin{document}

\twocolumn[
\icmltitle{Estimate Sequences for Variance-Reduced Stochastic Composite Optimization}



\icmlsetsymbol{equal}{*}

\begin{icmlauthorlist}
\icmlauthor{Andrei Kulunchakov}{inria}
\icmlauthor{Julien Mairal}{inria}
\end{icmlauthorlist}

\icmlaffiliation{inria}{Univ. Grenoble Alpes, Inria, CNRS, Grenoble INP, LJK, 38000 Grenoble, France}

\icmlcorrespondingauthor{Julien Mairal}{julien.mairal@inria.fr}
%\icmlcorrespondingauthor{Eee Pppp}{ep@eden.co.uk}

\icmlkeywords{stochastic optimization, variance reduction, acceleration}

\vskip 0.3in
]
\newcommand\red[1]{\textcolor{red}{#1}}



\printAffiliationsAndNotice{}  % leave blank if no need to mention equal contribution

\begin{abstract}
In this paper, we propose a unified view of gradient-based algorithms for
stochastic convex composite optimization by extending the concept of estimate
sequence introduced by Nesterov. 
This point of view covers the stochastic gradient descent method,
variants of the approaches SAGA, SVRG,
and has several advantages: (i) we provide a generic
proof of convergence for the aforementioned methods; (ii) we show that
this SVRG variant is adaptive to strong convexity;
(iii) we naturally
obtain new algorithms with the same guarantees; (iv) we derive generic
strategies to make these algorithms robust to stochastic noise, which is useful
when data is corrupted by small random perturbations. Finally, we show that
this viewpoint is useful to obtain new accelerated algorithms in the sense of Nesterov.

\end{abstract}

\section{Introduction}
We consider convex optimization problems of the form
\begin{equation}
\min_{x \in \Real^p} \left\{ F(x) \defin f(x)  + \psi(x) \right\}, \label{eq:prob}
\end{equation}
where $f$ is $\mu$-strongly convex and $L$-smooth (differentiable with $L$-Lipschitz continuous gradient),  
and~$\psi$ is convex lower-semicontinuous.
For instance, $\psi$ may be the $\ell_1$-norm but may also be the
indicator function of a convex set~$\Ccal$ for constrained problems~\citep[][]{hiriart_urruty_lemarechal_1993ii}.

More specifically, we focus on stochastic objectives, where~$f$ is an expectation or a finite
sum of convex functions
\begin{equation}
    f(x) = \E_\xi\left[\tilde{f}(x,\xi)\right] ~~~ \text{or} ~~~ f(x) = \frac{1}{n} \sum_{i=1}^n f_i(x). \label{eq:f}
\end{equation}
On the left, $\xi$ is a random variable representing a data point drawn according
to some distribution and $\tilde{f}(x,\xi)$ measures the fit of some model
parameter~$x$ to the data point~$\xi$. 

While the finite-sum setting is a particular case of expectation, the deterministic nature of the
resulting cost function drastically changes the performance guarantees an
optimization method may achieve to solve~(\ref{eq:prob}). In particular, when
an algorithm is only allowed to access unbiased measurements of the objective
and gradient, it may
be shown that the worst-case convergence rate in expected function value cannot
be better than $O(1/k)$ in general, where $k$ is the number of
iterations~\cite{nemirovski,agarwal2009information}.  

Even though this pessimistic result applies to the general stochastic case, linear
convergence rates can be obtained for deterministic finite sums.
For instance, linear rates are achieved by SAG \cite{schmidt2017minimizing}, SAGA \cite{saga},
SVRG \cite{johnson2013accelerating,proxsvrg}, SDCA \cite{accsdca}, MISO \cite{miso}, Katyusha \cite{accsvrg}, MiG \cite{MiG2018}, SARAH \cite{sarah}, accelerated SAGA \cite{zhou2018direct} or~\citet{conjugategradient}.
In the non-convex case, a recent focus has been on improving convergence rates for finding first-order stationary points~\cite{scsg2017,paquette2017catalyst,spider2018}, which is however beyond the scope of our paper.
A common interpretation is to see these algorithms as performing SGD steps with an estimate of the gradient that has lower variance.

In this paper, we are interested in providing a unified view of such stochastic optimization
algorithms, but we also want to investigate their \emph{robustness} to random
pertubations. Specifically, we consider objective functions with an explicit
finite-sum structure such as~(\ref{eq:f}) when only noisy
estimates of the gradients~$\nabla f_i(x)$ are available. 
Such a setting may occur for various reasons. Perturbations may be injected during training to achieve better
generalization~\cite{srivastava_dropout:_2014}, perform stable feature selection~\citep{meinshausen2010stability},
for privacy-aware learning~\citep{wainwright2012privacy} or to
improve model robustness~\cite{zheng2016improving}.

Each point indexed by~$i$ is then
corrupted by a random perturbation $\rho_i$ and the function~$f$ may
be written as
\begin{equation}
    f(x) =  \frac{1}{n} \sum_{i=1}^n f_i(x) ~~~\text{with}~~~ f_i(x) = \E_{\rho_i} \left[ \tilde{f}_i(x, \rho_i)\right]. \label{eq:f2}
\end{equation}
Since the exact gradients $\nabla f_i(x)$ cannot be computed, all the aforementioned variance-reduction
methods do not apply and the standard approach 
is to use SGD. Typically, the variance of the gradient estimate then
decomposes into two parts $\sigma^2=\sigma_s^2 + \tilde{\sigma}^2$, where
$\sigma_s^2$ is due to data sampling and
$\tilde{\sigma}^2$ to the random perturbation. In such a context,
variance reduction consists of designing algorithms whose convergence
rate depends on $\tilde{\sigma}^2$, which is potentially much smaller than $\sigma^2$.
The SAGA and SVRG methods have been adapted for such a purpose
by~\citet{hofmann_variance_2015}, though the resulting algorithms have non-zero
asymptotic error; the MISO method was adapted by~\citet{smiso} at the cost of a
memory overhead of $O(np)$, whereas other variants of SAGA and SVRG have been proposed
by~\citet{zheng2018lightweight} for linear models in machine learning.

In this paper, we extend estimate sequences introduced by \citet{nesterov},
which are typically used to design accelerated algorithms.
Estimate sequences have been used before for stochastic optimization
\citep{devolder_2011,lin_pena2014,Lu2015}, but not for the following generic purpose:

First, we make a large class of variance-reduction methods 
robust to stochastic perturbations. 
More precisely, by using a sampling strategy~$Q$ to select indices of the sum~(\ref{eq:f2}) at each iteration, 
when each function $f_i$ is convex and $L_i$-smooth,
the worst-case iteration complexity of our approaches in function values---that is, the number of iterations to guarantee $\E[F(x_k)-F^\star] \leq \varepsilon$---is upper bounded by
   \begin{equation*}
      O\left( \left(n + \frac{L_Q}{\mu}\right)\log\left(\frac{F(x_0)-F^\star}{\varepsilon} \right)  \right) + O\left( \frac{ \rho_Q\tilde{\sigma}^2}{\mu\varepsilon} \right),
   \end{equation*}
   where $\rho_Q \geq 1$ and $L_Q$ depends on $Q$. For the uniform distribution, we have $\rho_Q=1$ and
   $L_Q=\max_i L_i$, whereas a non-uniform $Q$ may yield $L_Q=\frac{1}{n}\sum_i L_i$.
   The term on the left corresponds to the
   complexity of variance-reduction methods
   without perturbation, and
   $O(\tilde{\sigma}^2/\mu \varepsilon)$ is the optimal sublinear rate of convergence
   for a stochastic optimization problem when the gradient estimates have
   variance~$\tilde{\sigma}^2$. In contrast, a variant of SGD
   applied to~(\ref{eq:f2}) has worst-case
   complexity $O(\sigma^2/\mu \varepsilon)$, with potentially
   $\sigma^2 \gg \tilde{\sigma}^2$. 

Second, we design a new accelerated algorithm which, to our knowledge, is the first one to achieve the
   complexity
\begin{equation*}
   O\left( \left(n + \sqrt{n\frac{L_Q}{\mu}}\right)\log\left(\frac{F(x_0)-F^\star}{\varepsilon} \right)  \right) + O\left( \frac{\rho_Q \tilde{\sigma}^2}{\mu\varepsilon} \right), 
\end{equation*}
where the term on the left 
matches the optimal complexity for finite
sums when $\tilde{\sigma}^2=0$~\cite{arjevani2016dimension}, which has been
achieved by~\citet{accsvrg,MiG2018,zhou2018direct,kovalev2019don}. The general
stochastic finite-sum problem with $\tilde{\sigma}^2>0$ was also considered with 
an acceleration mechanism by~\citet{lan_zhou_distributed2018} for 
distributed optimization being optimal in terms of communication rounds, but not in the global complexity. 

Note that we treat here only the strongly convex case, but similar
results can be obtained when $\mu=0$, as shown in a long version of this paper~\citep{kulunchakov2019estimate}.

\section{A Generic Framework}\label{sec:lower}
In this section, we introduce stochastic estimate sequences and show how they can handle variance reduction.

\subsection{A Classical Iteration Revisited}\label{subsec:dk}
Consider an algorithm that performs the following updates:

\custombox{
\begin{equation}
x_{k} \leftarrow \text{Prox}_{\eta_k\psi}\left[ x_\kmone - \eta_k g_k\right], \tag{\textsf{A}} \label{eq:opt1}
\vspace*{0.1cm}
\end{equation}
}%
where $\E[g_k | {\mathcal F}_\kmone] = \nabla f(x_\kmone)$ and
$\Fcal_{\kmone}$ is the filtration representing all information up to iteration $\kmone$, $\eta_k > 0$ is a step size, and $\text{Prox}_{\eta \psi}[.]$ is the proximal operator~\citep{moreau1962fonctions} defined for any scalar $\eta > 0$ as the unique solution of
\begin{equation}
\text{Prox}_{\eta \psi}[u] \defin \argmin_{x \in \Real^p} \left\{ \eta \psi (x) + \frac{1}{2}\|x-u\|^2 \right\}. \label{eq:prox}
\end{equation}
Key to our analysis, we interpret~(\ref{eq:opt1}) as the iterative minimization of quadratic surrogate functions.

\vs
\paragraph{Interpretation with stochastic estimate sequence.}~\\
Consider
\begin{equation}
 d_0(x) =  d_0^\star + \frac{\gamma_0}{2}\|x-x_0\|^2, \label{eq:d0}
\end{equation}
with $\gamma_0 \geq \mu$ and $d_0^\star$ is a scalar value that is left unspecified at the moment. Then, it is easy to show that $x_k$ in~(\ref{eq:opt1}) minimizes  $d_k$ defined recursively for $k \geq 1$  as
\begin{equation}
   d_k(x) = (1-\delta_k)d_\kmone(x) + \delta_k l_k(x),\label{eq:surrogate1}
\end{equation}
where 
\begin{multline*}
   l_k(x) =  f(x_\kmone) + g_k^\top( x-x_\kmone)  \\ + \frac{\mu}{2}\|x-x_\kmone\|^2 + \psi(x_k) + \psi'(x_k)^\top(x-x_{k}),
\end{multline*}
$\delta_k,\gamma_k$ satisfy the system of equations
\begin{equation}
    \delta_k = \eta_k\gamma_k ~~~\text{and}~~~ 
    \gamma_k = (1-\delta_k)\gamma_\kmone +\mu \delta_k,
\label{eq:gamma}
\end{equation}
(note that $\gamma_0=\mu$ yields $\gamma_k=\mu$ for all $k$), and
\begin{displaymath}
   \psi'(x_k)  = \frac{1}{\eta_k} (x_{\kmone} - x_k) - g_k.
\end{displaymath}
Here, 
$\psi'(x_k)$ is a subgradient in $\partial \psi(x_k)$. By simply using the definition of the
proximal operator~(\ref{eq:prox}) and considering first-order optimality conditions, we indeed have that $0
\in x_k - x_{\kmone} + \eta_k g_k + \eta_k \partial
\psi(x_k)$ and~$x_k$ coincides with the minimizer of~$d_k$. This allows us to write $d_k$ in the form
\begin{displaymath}
    d_k(x) = d_k^\star + \frac{\gamma_k}{2}\|x-x_k\|^2~~~\text{for all}~k \geq 0.  
\end{displaymath}
The construction~(\ref{eq:surrogate1}) is akin to that of estimate sequences
introduced by~\citet{nesterov}, which are typically used for
designing accelerated gradient-based optimization algorithms. In this section,
we are however not interested in acceleration, but instead in stochastic optimization
and variance reduction. One of the main properties of estimate sequences that we will
nevertheless use is their ability to behave asymptotically as a lower bound. Indeed, we have
\begin{equation}
    \E[d_k(x^\star)]  
      \leq \Gamma_k d_0(x^\star) + (1-\Gamma_k)F^\star,
\label{eq:est}
\end{equation}
where $x^\star$ is a minimizer of $F$ and $\Gamma_k = \prod_{t=1}^k (1-\delta_t)$.
The inequality comes from strong convexity since $\E[g_k^\top (x^\star-x_\kmone)|\Fcal_\kmone] = \nabla f(x_\kmone)^\top (x^\star-x_\kmone)$, leading to the relation 
$\E[d_k(x^\star)]  \leq (1-\delta_k)\E[d_\kmone(x^\star)] + \delta_k F^\star$.
Then, by unrolling the recursion, we obtain~(\ref{eq:est}).
When~$\Gamma_k$ converges to zero, the contribution of the initial surrogate $d_0$ disappears and $\E[d_k(x^\star)]$ behaves as a lower bound of $F^\star$.

\paragraph{Relation with existing algorithms.} The iteration~(\ref{eq:opt1})
encompasses many approaches such as ISTA (proximal gradient
uses the exact gradient $g_k = \nabla f(x_{\kmone})$ \citep{fista,nesterov2013gradient} or
proximal variants of the stochastic gradient descent method to deal with a
composite objective~\citep[][]{Lan2012}. 
Of interest to us, the variance-reduced stochastic optimization approaches
SVRG~\citep{proxsvrg} and SAGA~\cite{saga} also follow~(\ref{eq:opt1}) but with an unbiased gradient estimator~$g_k$ whose
variance reduces over time.  Specifically, they use
 \begin{equation}
    g_k = \nabla f_{i_k}(x_{\kmone}) - z^{i_k}_{\kmone} + \bar{z}_{\kmone} ~~\text{with}~~ \bar{z}_\kmone = \frac{1}{n} \sum_{i=1}^n z^{i}_\kmone, \label{eq:svrg}
 \end{equation}
where $i_k$ is an index chosen uniformly in $\{1,\ldots,n\}$ at random, and each
 $z^{i}_k$ is equal to a gradient $\nabla
f_i(\tilde{x}_k^i)$, where~$\tilde{x}_k^i$ is one of the previous iterates. 
The motivation is that given two random variables $X$ and $Y$,
it is possible to define a new variable $Z=X-Y + \E[Y]$ which has the same
expectation as~$X$ but potentially a lower variance if $Y$ is positively
correlated with $X$.  
SVRG uses the same anchor point $\tilde{x}_k^i = \tilde{x}_k$ for all $i$, where $\tilde{x}_k$ is
updated every $m$ iterations. Typically, the memory cost of SVRG is that of
storing the variable $\tilde{x}_k$ and the gradient $\bar{z}_k = \nabla f(\tilde{x}_k)$, which is thus $O(p)$.
On the other hand, SAGA updates only $z^{i_k}_k = \nabla f_{i_k}({x}_{\kmone})$ at iteration~$k$,
such that $z^{i}_{k}  =  z^{i}_{\kmone}$ if $i \neq i_k$.
Thus, SAGA requires storing $n$ gradients. While in general the
overhead cost in memory is of order $O(np)$, it may be reduced to $O(n)$ when
dealing with linear models in machine learning~\citep[see][]{saga}.
Note that variants with non-uniform sampling of the indices $i_k$ have been proposed by~\citet{proxsvrg,saganonu},
which we discuss later.

In order to make our proofs consistent for SAGA and SVRG (and MISO in~\citealp{kulunchakov2019estimate}), we consider a variant of SVRG with a randomized gradient updating schedule~\cite{hofmann_variance_2015}. Remarkably, this variant was shown to provide benefits over the fixed schedule in a concurrent work~\cite{kovalev2019don} when $\tilde{\sigma}^2=0$. 

\subsection{Gradient Estimators and New Algorithms} \label{subsec:gradients}
In this paper, we consider the gradient estimators below. For all of them, we define the variance $\sigma_k$ to be
\begin{displaymath}
    \sigma_k^2 = \E\left[ \|g_k - \nabla f(x_\kmone)\|^2\right].
\end{displaymath}

\vs
\paragraph{ISTA.} Simply consider $g_k=\nabla f(x_\kmone)$ and $\sigma_k=0$.

\vs
\paragraph{SGD.} We assume that $g_k$ has variance bounded by $\sigma^2$. Typically, when $f(x) = \E_\xi[\tilde{f}(x,\xi)]$, a data point $\xi_k$ is drawn at iteration $k$ and $g_k = \nabla \tilde{f}(x,\xi_k)$. Even though the bounded variance assumption has limitations,
it remains the most standard  one for stochastic optimization and more realistic settings (such as 
\cite{bottou2018optimization,nguyen2018sgd} for the smooth case) are left for future work.

\vs
\paragraph{\bfseries random-SVRG.} For finite sums, we consider a variant of SVRG with random update of the anchor point~$\tilde{x}_\kmone$, proposed originally in~\cite{HofmannLM15}, combined with non-uniform sampling. Specifically, $g_k$ is defined as  
        \begin{equation}
            g_k = \frac{1}{q_{i_k} n}\left(\tildenabla f_{i_k} (x_\kmone) - z_{\kmone}^{i_k} \right) + \bar{z}_\kmone, \label{eq:gk}
        \end{equation}
        where $i_k \!\sim\! Q=\{q_1,\ldots,q_n\}$ and $\tildenabla$ denotes a perturbed gradient operator.
        For instance, if $f_i(x)=\E_\rho[\tilde{f}_i(x,\rho)]$ for all~$i$, where~$\rho$ is a stochastic perturbation,
        instead of accessing $\nabla f_{i_k}(x_\kmone)$, we draw a perturbation $\rho_k$ and observe
        \begin{displaymath}
             \tildenabla f_{i_k}(x_\kmone) = \nabla \tilde{f}_{i_k}(x_\kmone,\rho_k)  = \nabla f_{i_k} (x_\kmone) + {\zeta_k},
        \end{displaymath}
        where the perturbation $\zeta_k$ has zero mean given $\Fcal_\kmone$ and its variance is bounded by~$\tilde{\sigma}^2$. When considering the setting without perturbation, we simply have $\tildenabla = \nabla$.

        Similar to the previous case, the variables $z_k^{i}$ and $\bar{z}_k$ also correspond to noisy estimates of the gradients. Specifically,
        \begin{displaymath}
            z_{k}^{i}  = \tildenabla f_{i}(\tilde{x}_k)  ~~~~\text{and}~~~ \bar{z}_k = \frac{1}{n}\sum_{i=1}^n z_k^i,
        \end{displaymath}
        where $\tilde{x}_k$ is an anchor point that is updated on average every~$n$ iterations. 
        Whereas the classical SVRG approach
        updates~$\tilde{x}_k$ on a fixed schedule, we perform random
        updates: with probability $1/n$, we choose $\tilde{x}_k = x_k$ and
        recompute~$\bar{z}_k=\tildenabla f(\tilde{x}_k)$; otherwise $\tilde{x}_k$ is
        kept unchanged. In comparison with the fixed schedule, the analysis with the random one is
        simplified and can be unified with that of SAGA. This approach is described in Algorithm~\ref{alg:svrg}.

        In terms of memory, the random-SVRG gradient estimator requires to
        store an anchor point~$\tilde{x}_\kmone$ and the average gradients
        $\bar{z}_\kmone$. The $z_k^i$'s do not need to be stored; only the~$n$
        random seeds to produce the perturbations are kept into memory,
        which allows us to compute
        $z^{i_k}_\kmone = \tildenabla f_{i_k}(\tilde{x}_\kmone)$ at
        iteration~$k$, with the same perturbation for index~$i_k$ that was used
        to compute $\bar{z}_\kmone = \frac{1}{n}\sum_{i=1}^n z_\kmone^i$ when the anchor point was last updated.
        The overall cost is thus $O(n+p)$.

\begin{algorithm}[ht!]
\caption{Iteration~(\ref{eq:opt1}) with random-SVRG estimator}\label{alg:svrg}
\begin{algorithmic}[1]
 \STATE {\bfseries Input:} $x_0$ in $\Real^p$ (initial point); $K$ (number of iterations); $(\eta_k)_{k \geq 0}$ (step sizes); $\gamma_0 \geq \mu$ (if averaging);
 \STATE {\bfseries Initialization:} $\tilde{x}_0 = \hat{x}_0 = x_0$; $\bar{z}_0=\frac{1}{n}\sum_{i=1}^n \tildenabla f_i(\tilde{x}_0)$;
\FOR{$k=1,\ldots,K$}
    \STATE Sample $i_k$ according to $Q=\{q_1,\ldots,q_n\}$;
     \STATE Compute the gradient estimator with perturbations:
     \vs
        \begin{equation*}
           g_k = \frac{1}{q_{i_k} n}\left(\tildenabla f_{i_k} (x_\kmone) - \tildenabla f_{i_k}(\tilde{x}_\kmone) \right) + \bar{z}_\kmone; 
        \end{equation*}
     \STATE Compute the next iterate
     \vs
     $$x_{k} \leftarrow \text{Prox}_{\eta_k\psi}\left[ x_\kmone - \eta_k g_k\right];$$
     \STATE With probability $1/n$, 
     \vs
     \begin{displaymath}
         \tilde{x}_k = x_k~~~~\text{and}~~~~ \bar{z}_k =\frac{1}{n}\sum_{i=1}^n\tildenabla f_i(\tilde{x}_k);
     \end{displaymath}
     \vs
     \STATE Otherwise, with probability $1-1/n$, keep the anchor point unchanged $\tilde{x}_k = \tilde{x}_\kmone$ and $\bar{z}_k = \bar{z}_\kmone$;
     \STATE {\bfseries Optional}: Use online averaging using $\delta_k$ from~(\ref{eq:gamma}):
     \vs
     $$\hat{x}_k =
     (1-\tau_k) \hat{x}_{\kmone} + \tau_k x_k~~\text{with}~~ \tau_k = \min\left( \delta_k, \frac{1}{5n}\right);$$
     \vs
\ENDFOR
\STATE {\bfseries Output:} $x_K$ or $\hat{x}_K$ if averaging.
\end{algorithmic}
\end{algorithm}

%\vspace*{-0.5cm}
\vs
\paragraph{\bfseries SAGA.} The estimator has a form similar to~(\ref{eq:gk}) but with a different choice of variables
        $z_k^i$.  Unlike SVRG that stores an anchor
        point~$\tilde{x}_k$, the SAGA estimator requires storing and
        incrementally updating the $n$ auxiliary variables $z_k^i$, while maintaining the relation $\bar{z}_k =
        \frac{1}{n}\sum_{i=1}^n z_k^i$. 
        We consider variants such that each gradient~$\nabla f_i(x)$ is corrupted by a random perturbation;
        to deal with non-uniform sampling, we use a similar strategy
        as~\citet{saganonu}. 
        The corresponding algorithm is available in Appendix~\ref{appendix:saga}.

\addtolength{\textfloatsep}{0.2cm}
\section{Convergence Analysis and Robustness}\label{sec:theory}
In
Section~\ref{subsec:conv1}, we present a general convergence result and the analysis with variance-reduction is presented in
Section~\ref{subsec:conv2}. All proofs are in the appendix.

\subsection{Generic Convergence Result}\label{subsec:conv1}
The following proposition gives a key relation between $F(x_k)$, the surrogate $d_k$, $d_\kmone$ and the variance $\sigma_k$.

\begin{proposition}[\bf Key relation]\label{prop:keyprop}
   For iteration~(\ref{eq:opt1}), assuming $\eta_k \leq 1/L$, we have for all $k \geq 1$,
\begin{multline}
    \delta_k (\E[F(x_k)]-F^\star) +\E[d_k(x^\star)-d_k^\star]  \\ \leq  (1-\delta_k)\E[d_\kmone(x^\star)-d_{\kmone}^\star] + \eta_k {\delta_k \sigma_k^2}, \label{eq:relation}
\end{multline}
where $x^\star$ is a minimizer of $F$ and $F^\star=F(x^\star)$.
\end{proposition}
Then, without making further assumption on~$\sigma_k$, we have the following general convergence result, which is a direct consequence of the averaging Lemma~\ref{lemma:averaging} in the appendix, inspired in part by~\citet{ghadimi2012optimal}:
\begin{theorem}[\bf General convergence result]\label{thm:conv}
Under the same assumptions as in Proposition~\ref{prop:keyprop}, 
by using the averaging strategy of Lemma~\ref{lemma:averaging}, which produces an iterate~$\hat{x}_k$, 
\begin{multline*}
   \E\left[F(\hat{x}_k)- F^\star + \frac{\gamma_k}{2}\|x_k-x^\star\|^2\right] \\ \leq  {\Gamma_k}\left(F(x_0)- F^\star + \frac{\gamma_0}{2}\|x_0-x^\star\|^2  +  \sum_{t=1}^k \frac{\delta_t\eta_t\sigma_t^2}{\Gamma_t} \right), 
\end{multline*}
where $\Gamma_k = \prod_{t=1}^k (1-\delta_t)$ and $x^\star$ is a minimizer of $F$.
\end{theorem}
Theorem~\ref{thm:conv} allows us to recover convergence rates for various algorithms.
In the corollary below, we consider the stochastic setting with constant step sizes; the
algorithm converges with the same rate as the deterministic problem to a
noise-dominated region of radius $\sigma^2/L$. 
The proof simply uses Lemma~\ref{eq:rate_gamma}, which provides the convergence rate of $(\Gamma_k)_{k \geq 0}$ and uses the relation $\Gamma_k \sum_{t=1}^k \frac{\delta_t}{\Gamma_t} =1 -\Gamma_k \leq 1$ from Lemma~\ref{lemma:simple} in the appendix.
\begin{corollary}[\bf Prox-SGD with constant step-size]~\label{corollary:sgd_constant}\newline
Assume in Theorem~\ref{thm:conv} that $\sigma_k\leq \sigma$, and choose $\gamma_0=\mu$ and $\eta_k = 1/L$. Then, 
\begin{multline}
   \E\left[F(\hat{x}_k)- F^\star + \frac{\mu}{2}\|x_k-x^\star\|^2\right] \\ \leq   \left(1-\frac{\mu}{L}\right)^k\left(F(x_0)- F^\star + \frac{\mu}{2} \|x_0 - x^\star\|^2\right) + \frac{\sigma^2}{L}.\label{eq:sgd}
\end{multline}
\end{corollary}
We now show that it is also possible to obtain converging algorithms  
by using decreasing step sizes. 
\begin{corollary}[\bf Prox-SGD with decreasing step-sizes] \label{corollary:sgd}~\newline
Assume that we target an accuracy
$\varepsilon$ smaller than $2\sigma^2/L$. First, use iteration~(\ref{eq:opt1}) as in Theorem~\ref{thm:conv} with a constant step-size $\eta_k=1/L$ and $\gamma_0=\mu$, 
leading to the convergence rate~(\ref{eq:sgd}),  until $\E[F(\hat{x}_k)- F^\star] \leq  2 \sigma^2/L$.
Then, restart the optimization procedure with decreasing step-sizes $\eta_k
= \min \left(\frac{1}{L},\frac{2}{\mu (k+2)}\right)$. The resulting number of gradient evaluations to achieve $\E[F(\hat{x}_k)- F^\star] \leq  \varepsilon$ is upper bounded by
$$O\left( \frac{L}{\mu} \log\left(\frac{F(x_0)- F^\star}{\varepsilon}\right)\right) + O\left( \frac{\sigma^2}{\mu \varepsilon}\right).$$
\end{corollary}
We note that the dependency in $\sigma^2$ with the rate $O(\sigma^2/\mu \varepsilon)$ is optimal for strongly convex functions~\cite{nemirovski}. 
Unfortunately, estimating $\sigma$ is not easy and knowing in practice when to start decreasing the step sizes in SGD algorithms is an open problem.
The corollary simply supports the heuristic consisting of adopting a constant step size long enough until the iterates oscillate without much progress, before decreasing the step sizes~\citep[see][]{bottou2018optimization}.

\subsection{Faster Convergence with Variance Reduction}\label{subsec:conv2}

Stochastic variance-reduced algorithms rely on gradient estimates whose variance
decreases as fast as the objective. Our framework provides a unified proof of
convergence for our variants of SVRG and SAGA and  makes them robust to stochastic perturbations. 
Specifically, we consider the minimization
of a finite sum of functions as in~(\ref{eq:f2}), but each observation of the
gradient $\nabla f_i(x)$ is corrupted by a random noise variable.
The next proposition extends the proof of SVRG~\cite{proxsvrg} and characterizes~$\sigma_k^2$. 
\begin{proposition}[\bf Generic Upper-Bound on Variance]\label{prop:nonu}~\newline
Consider the optimization problem~(\ref{eq:prob}) when $f$ is a finite sum of functions $f=\frac{1}{n}\sum_{i=1}^n f_i$ where each $f_i$ is convex and $L_i$-smooth with $L_i \geq \mu$. Then, the random-SVRG and SAGA gradient estimates defined in Section~\ref{subsec:gradients} satisfy
\begin{multline}
   \sigma_k^2 \leq 4 L_Q \E[F(x_{\kmone})  - F^\star]  \\ + \frac{2}{n}\E\left[ \sum_{i=1}^n \frac{1}{n q_i}\| u^i_\kmone - \nabla f_i(x^\star)\|^2\right] + {3 \rho_Q \tilde{\sigma}^2}, \label{eq:var2}
\end{multline}
where $L_Q = \max_i L_i/(q_i n)$, $\rho_Q = 1/(n \min_i q_i)$, and for all $i$ and $k$, $u_k^i$ is equal to $z_k^i$ without noise---that is 
        \begin{displaymath}
        \begin{split}
u_k^i & = \nabla f_i(\tilde{x}_k) ~~~\text{for random-SVRG}\\
            u_{k}^{j_k} & = \nabla f_{j_k}(x_k) ~~~\text{and}~~~ u_k^j  = u_\kmone^j ~~~\text{if}~~~j \neq j_k ~~~\text{for SAGA}.
            \end{split}
        \end{displaymath}
\end{proposition}
Next, we apply this result to Proposition~\ref{prop:keyprop}.
\begin{proposition}[\bf Lyapunov function]\label{thm:lyapunov}
Consider the same setting as Proposition~\ref{prop:nonu} and the same gradient estimators.
When using the construction of $d_k$ from Sections~\ref{subsec:dk}, and assuming $\gamma_0 \geq \mu$ and $(\eta_k)_{k \geq 0}$ is non-increasing with $\eta_k \leq \frac{1}{12 L_Q}$, we have for all $k \geq 1$, with $\tau_k = \min\left(\delta_k, \frac{1}{5n}\right)$,
\begin{multline}
    \frac{\delta_k}{6}\E[F(x_k)-F^\star] + T_k  \leq \left( 1 - \tau_k \right)T_\kmone  + {3 \rho_Q \eta_k \delta_k\tilde{\sigma}^2}, \label{eq:aux2} \\
\text{where}~~T_k = {5} L_Q\eta_k \delta_k \E[F(x_k)-F^\star] + \E[d_k(x^\star)-d_k^\star] \\ + \frac{5 \eta_k \delta_k}{2}\E\left[\frac{1}{n} \sum_{i=1}^n \frac{1}{q_i n} \| u^i_k - u^i_\star\|^2\right].
\end{multline}
\end{proposition}
From the Lyapunov function, we obtain a general convergence result for the variance-reduced stochastic algorithms.
\begin{theorem}[\bf Convergence with variance-reduction]\label{thm:svrg}
Consider the same setting as Proposition~\ref{thm:lyapunov}. Then, by using the averaging strategy described in Algorithm~\ref{alg:svrg},
\begin{multline*}
   \E\left[F(\hat{x}_k)-F^\star  + \frac{6 \tau_k}{\delta_k} T_k\right]  \\ \leq \Theta_k \left( F(x_0)-F^\star + \frac{6 \tau_k}{\delta_k} T_0 + \frac{18 \rho_Q \tau_k \tilde{\sigma}^2}{\delta_k} \sum_{t=1}^k \frac{\eta_t \delta_t}{\Theta_t}     \right),
\end{multline*}
where $\Theta_k = \prod_{t=1}^k(1-\tau_t)$.
\end{theorem}
The theorem is a direct application of the averaging Lemma~\ref{lemma:averaging} to Proposition~\ref{thm:lyapunov}.
From this generic convergence theorem, we now study particular cases.
\begin{corollary}[\bf Variance-reduction with constant $\eta$]\label{corollary:svrg0}~\newline
   Consider the same setting as in Theorem~\ref{thm:svrg}, with $\gamma_0=\mu$, $\eta_k = \frac{1}{12L_Q}$, and $\tau_k = \tau = \min\left( \frac{\mu}{12L_Q}, \frac{1}{5n}\right)$. Then,
\begin{multline*}
   \E\left[F(\hat{x}_k)-F^\star  + 36L_Q \tau \|x_k-x^\star\|^2\right] \\ \leq  8 \Theta_k\left( F(x_0)-F^\star \right)  + \frac{3 \rho_Q \tilde{\sigma}^2}{2L_Q}. 
\end{multline*}
\end{corollary}
This first corollary shows that the algorithm achieves a linear convergence rate to a noise-dominated region. Interestingly, the algorithm \emph{without averaging} does not require computing $\tau$ and produces iterates $(x_k)_{k \geq 0}$ without using the strong convexity constant~$\mu$. \emph{This shows that all estimators we consider can become adaptive to~$\mu$}. 

Moreover, we note that the non-uniform strategy slightly degrades the dependency in $\tilde{\sigma}^2$: indeed, $\rho_Q=1$ and $L_Q/\rho_Q = \max_{i} L_i$ if~$Q$ is uniform, but with non-uniform $q_i = L_i/\sum_{j=1}^n L_j$, we have instead $L_Q = \frac{1}{n}\sum_{i=1}^nL_i$ (which is better) and $L_Q/\rho_Q = \min_{i} L_i$ (which is worse).

\begin{corollary}[\bf Variance-reduction with decreasing $\eta_k$]\label{corollary:svrg2}
Consider the same setting as in Theorem~\ref{thm:svrg} and target an accuracy $\varepsilon \leq  {24 \rho_Q \eta\tilde{\sigma}^2}$,
with $\eta = \min\left(\frac{1}{12L_Q}, \frac{1}{5\mu n}\right)$.
Then, use a constant step-size strategy $\eta_k = \eta$ with $\gamma_0=\mu$ until we find a point $\hat{x}_k$ such that $\E[F(\hat{x}_k)-F^\star] \leq {24\rho_Q \eta \tilde{\sigma}^2}$.
Then, restart the optimization with decreasing step-sizes 
$\eta_k = \min \left({\eta},\frac{2}{\mu (k+2)}\right)$. The
number of gradient evaluations to achieve $\E[F(\hat{x}_k)- F^\star] \leq  \varepsilon$
is upper bounded by
 $$O\left( \left(n + \frac{L_Q}{\mu}\right) \log\left(\frac{F(x_0)-F^\star}{\varepsilon} \right)  \right) + O\left( \frac{\rho_Q \tilde{\sigma}^2}{\mu \varepsilon}\right).$$
\end{corollary}
The corollary shows that variance-reduction algorithms may exhibit an optimal dependency on the noise level $\tilde{\sigma}^2$.

\section{Accelerated Stochastic Algorithms}\label{sec:acc}
We now consider the following iteration, involving an extrapolation sequence $(y_k)_{k \geq 1}$, which is a classical mechanism from accelerated first-order algorithms~\cite{fista,nesterov2013gradient}. Given a sequence of step-sizes $(\eta_k)_{k \geq 0}$ with $\eta_k \leq 1/L$  for all $k \geq 0$, and $\gamma_0 \geq \mu$, we consider the sequences $(\delta_k)_{k \geq 0}$ and $(\gamma_k)_{k \geq 0}$ that satisfy 
\begin{displaymath}
\begin{split}
    \delta_k & = \sqrt{\eta_k \gamma_k} ~~~\text{for all}~k \geq 0 \\
    \gamma_k & = (1-\delta_k) \gamma_\kmone + \delta_k \mu ~~~\text{for all}~k \geq 1.\\
\end{split}
\end{displaymath}
Then, for $k \geq 1$, we consider the iteration
\custombox{
\vspace*{-0.2cm}
\begin{equation}
\begin{split}
x_{k}  & = \text{Prox}_{\eta_k\psi}\left[ y_\kmone - \eta_k g_k\right]\label{eq:opt3} \\
y_k & = x_k \!+\! \beta_k(x_k\! -\! x_\kmone)~~\text{with}~ \beta_k = \frac{\delta_k(1-\delta_k)\eta_{k+1}}{\eta_k \delta_{k+1} + \eta_{k+1}\delta_k^2},
\end{split}
\tag{\textsf{B}} 
\vspace*{-0.2cm}
\end{equation}
}
where $\E[g_k | {\mathcal F}_\kmone] = \nabla f(y_\kmone)$.
Iteration~(\ref{eq:opt3}) resembles the accelerated SGD approaches of~\citet{kwok2009,ghadimi2012optimal,lin_pena2014} but is slightly simpler since it involves two sequences of variables instead of three.

\subsection{Convergence Analysis without Variance Reduction}
Consider then the stochastic estimate sequence $d_k$ introduced in~(\ref{eq:surrogate1})
with $d_0$ defined as in~(\ref{eq:d0}) and
\begin{multline}
   l_k(x) = f(y_\kmone) + g_k^\top (x - y_\kmone) \\ + \frac{\mu}{2}\|x-y_\kmone\|^2 + \psi(x_k) + \psi'(x_k)^\top (x-x_k), \label{eq:lk}
\end{multline}
and $\psi'(x_k)=\frac{1}{\eta_k}(y_\kmone-x_k) - g_k$ is in $\partial \psi(x_k)$ by definition of the proximal operator.
As in Section~\ref{sec:lower}, $d_k(x^\star)$ asymptotically becomes a lower bound on $F^\star$ since~(\ref{eq:est}) remains satisfied.
This time, the iterate $x_k$ does not minimize $d_k$, and we denote by $v_k$ instead its minimizer, allowing us to write $d_k$ in the canonical form
\begin{displaymath}
   d_k(x) = d_k^\star + \frac{\gamma_k}{2}\|x- v_k\|^2.
\end{displaymath}
The first lemma highlights classical relations between the iterates $(x_k)_{k
\geq 0}$, $(y_k)_{k \geq 0}$ and the minimizers~$(v_k)_{k \geq 0}$, which also appears in~\cite[][p. 78]{nesterov} for constant $\eta_k$. 
Note that the construction of stochastic estimate sequence resembles that of~\cite{devolder_2011,lin_pena2014}. The main difference lies in the choice of function $l_k$ in~(\ref{eq:lk}), which yields a different algorithm and slightly stronger guarantees.
 \begin{lemma}[\bf Relations between $y_k$ and $x_k$]\label{lemma:acc}
 The sequences $(x_k)_{k \geq 0}$ and $(y_k)_{k \geq 0}$ produced by
iteration~(\ref{eq:opt3}) satisfy for all $k \geq 0$, with $v_0=y_0=x_0$, 
 \begin{displaymath}
 \begin{split}
    y_k & = \theta_k x_k + (1-\theta_k) v_k ~~~~\text{with}~~~~ \theta_k = \frac{\gamma_{k+1}}{ \gamma_k + \delta_{k+1}\mu}.
\end{split}
 \end{displaymath}
\end{lemma}

Then, the next lemma will be used to prove that $\E[F(x_k)] \leq \E[d_k^\star] + \xi_k$, where $\xi_k$ is a noise term, such that $\E[F(x_k)]-F^\star \leq \Gamma_k (d_0(x^\star) - F^\star) + \xi_k$, according to~(\ref{eq:est}).
\begin{lemma}[\bf Key lemma for acceleration]\label{lemma:key_acc}
Consider the same sequences as in Lemma~\ref{lemma:acc}. Then, for all $k \geq 1$,
   \begin{displaymath}
   \E[F(x_k)]  \leq \E\left[l_k(y_\kmone)\right] + \left(\frac{L\eta_k^2}{2} \! - \! \eta_k\right)\E\left[\|\tilde{g}_k\|^2\right]  +  \eta_k\sigma_k^2,
   \end{displaymath}
   with $\sigma_k^2= \E[\|\nabla f(y_\kmone)-g_k\|^2]$ and $\tilde{g}_k=g_k + \psi'(x_k)$.
\end{lemma}
Finally, we obtain the following convergence result.
\begin{theorem}[\bf Convergence of accelerated SGD]\label{thm:acc_sgd}
Under the assumptions of Lemma~\ref{lemma:acc}, we have for all $k \geq 1$,
\begin{multline*}
\E\left[F(x_k) - F^\star + \frac{\gamma_k}{2}\|v_k-x^\star\|^2\right] \\ \leq \Gamma_k \left(F(x_0)-F^\star + \frac{\gamma_0}{2}\|x_0-x^\star\|^2 + \sum_{t=1}^k \frac{\eta_t \sigma_t^2}{\Gamma_t}\right), 
\end{multline*}
where, as before, $\Gamma_t = \prod_{i=1}^t (1-\delta_i)$. 
\end{theorem}
We now specialize the theorem to various practical cases. 
\begin{corollary}[\bf Prox accelerated SGD with constant $\eta$]~\label{corollary:acc_sgd2a}\newline
Assume that~$g_k$ has constant variance $\sigma_k=\sigma$, and choose $\gamma_0=\mu$ and $\eta_k = 1/L$ with Algorithm~(\ref{eq:opt3}). Then, 
\begin{multline*}
   \E\left[F(x_k)- F^\star\right] \\ \leq   \left(1-\sqrt{\frac{\mu}{L}}\right)^k\left(F(x_0)- F^\star + \frac{\mu}{2} \|x_0 - x^\star\|^2\right) + \frac{\sigma^2}{\sqrt{\mu L}}. 
\end{multline*}
\end{corollary}
We now show that with decreasing step sizes, we obtain an algorithm with optimal complexity similar to~\citep{ghadimi2013optimal,cohen2018acceleration,aybat2019universally}, though we use a two-stages algorithm only.
\begin{corollary}[\bf Prox accelerated SGD with decreasing~$\eta_k$] \label{corollary:acc_sgd2}
Target an accuracy
$\varepsilon$ smaller than $2\sigma^2/\sqrt{\mu L}$. First, use a constant step-size $\eta_k=1/L$ with $\gamma_0=\mu$ within Algorithm~(\ref{eq:opt3})  until $\E[F({x}_k)- F^\star] \leq  2 \sigma^2/\sqrt{\mu L}$.
Then, we restart the optimization procedure with decreasing step-sizes $\eta_k
= \min \left(\frac{1}{L},\frac{4}{\mu (k+2)^2}\right)$. The number of gradient evaluations to achieve $\E[F({x}_k)- F^\star] \leq  \varepsilon$ is upper bounded by
$$O\left( \sqrt{\frac{L}{\mu}} \log\left(\frac{F(x_0)- F^\star}{\varepsilon}\right)\right) + O\left( \frac{\sigma^2}{\mu \varepsilon}\right).$$
\end{corollary}

\subsection{Accelerated Algorithm with Variance Reduction}
Next, we show how to build accelerated algorithms with the random-SVRG gradient estimator. 
First, we control the variance of the estimator in a similar manner to Katyusha~\cite{accsvrg}, as stated in the next proposition. Note that the estimator here does not require storing the seed of the random perturbations, unlike in the previous section, and does not rely on an averaging procedure (hence preserving the potential sparsity of the solution when $\psi$ is sparsity-inducing).

 \begin{algorithm}[ht!]
 \caption{Accelerated and robust random-SVRG}\label{alg:acC_svrg}
 \begin{algorithmic}[1]
  \STATE {\bfseries Input:} $x_0$ in $\Real^p$ (initial point); $K$ (number of iterations); $(\eta_k)_{k \geq 0}$ (step sizes); $\gamma_0 \geq \mu$;
  \STATE {\bfseries Initialization:} $\tilde{x}_0 = v_0 = x_0$; $\bar{z}_0= \tildenabla f(x_0)$;
 \FOR{$k=1,\ldots,K$}
      \STATE Find $(\delta_k,\gamma_k)$ such that
      \vsmall
      \begin{displaymath}
          \gamma_k = (1-\delta_k)\gamma_\kmone + \delta_k \mu ~~~~\text{and}~~~\delta_k = \sqrt{\frac{5\eta_k\gamma_k}{3n}};
      \end{displaymath}
      \STATE Choose 
      \vsmall
      \begin{displaymath}
        \!\! y_\kmone \!=\!  \theta_k v_\kmone + (1-\theta_k) \tilde{x}_\kmone~~\text{with}~~ \theta_k \!=\! \frac{3 n \delta_k \!-\! 5 \mu \eta_k}{3 \!-\! 5\mu \eta_k};
      \vsmall
      \end{displaymath}
      \STATE Sample $i_k \sim Q=\{q_1,\ldots,q_n\}$;
      \STATE Compute the gradient estimator with perturbations:
      \vsmall
         \begin{equation*}
            g_k = \frac{1}{q_{i_k} n}\left(\tildenabla f_{i_k} (y_\kmone) - \tildenabla f_{i_k}(\tilde{x}_\kmone) \right) + \bar{z}_\kmone; %\label{eq:svrg4}
      \vsmall
         \end{equation*}
      \STATE Obtain the new iterate
      \vsmall
      \begin{displaymath}
      x_{k} \leftarrow \text{Prox}_{\eta_k\psi}\left[ y_\kmone - \eta_k g_k\right];
      \vsmall
      \end{displaymath}
      \vsmall
      \STATE Find the minimizer $v_k$ of the estimate sequence $d_k$:
      \vsmall
      \begin{displaymath}
          v_k = \left(1- \frac{\mu \delta_k}{\gamma_k}\right)v_\kmone + \frac{\mu\delta_k}{\gamma_k}y_\kmone + \frac{\delta_k}{\gamma_k \eta_k}(x_k-y_\kmone);
      \vsmall
      \end{displaymath}
      \STATE With probability $1/n$, update the anchor point
      \vsmall
      \begin{displaymath}
          \tilde{x}_k = x_k ~~~~\text{and}~~~~ \bar{z}_k = \tildenabla f(\tilde{x}_k);
      \vsmall
      \end{displaymath}
      \STATE Otherwise, with probability $1-1/n$, keep the anchor point unchanged $\tilde{x}_k = \tilde{x}_\kmone$ and $\bar{z}_k = \bar{z}_\kmone$;
 \ENDFOR
 \STATE {\bfseries Output:} $\tilde{x}_k$.
 \end{algorithmic}
 \end{algorithm}

\begin{proposition}[\bf Variance reduction for random-SVRG]\label{prop:nonu2}
Consider problem~(\ref{eq:prob}) when $f$ is a finite sum of
functions $f=\frac{1}{n}\sum_{i=1}^n f_i$ where each $f_i$ is
$L_i$-smooth with $L_i \geq \mu$. 
Then, the variance of $g_k$ defined in Algorithm~\ref{alg:acC_svrg}
satisfies
\begin{displaymath}
    \sigma_k^2 \leq {2 L_Q}\left[  f(\tilde{x}_\kmone) \!-\! f(y_{\kmone}) \!-\!  g_k^\top (\tilde{x}_\kmone\!-\!y_\kmone)\right]  \!+\! {3 \rho_Q \tilde{\sigma}^2}.
\end{displaymath}
\end{proposition}
Then, we extend Lemma~\ref{lemma:key_acc} to the variance-reduction setting.
\begin{lemma}[\bf Key for accelerated variance-reduction]\label{lemma:key_acc_svrg}~\newline
   Consider the iterates provided by Algorithm~\ref{alg:acC_svrg} and 
call $a_k = 2L_Q \eta_k$ and $\tilde{g}_k=g_k + \psi'(x_k)$. Then,
   \begin{multline*}
   \E[F(x_k)]   \leq  \E\left[a_k F(\tilde{x}_\kmone) + (1- a_k) l_k(y_\kmone)\right] \\
     + \E\left[a_k \tilde{g}_k^\top (y_\kmone \!-\! \tilde{x}_\kmone) \!+\! \left(\frac{L\eta_k^2}{2} \!-\! \eta_k\right)\|\tilde{g}_k\|^2\right] + {3\rho_Q \eta_k\tilde{\sigma}^2}.
   \end{multline*}
\end{lemma}
Then, we may now state our main convergence result.
\begin{theorem}[\bf Convergence of the accelerated SVRG]\label{thrm:acc_svrg}
   Consider the iterates provided by Algorithm~\ref{alg:acC_svrg} and assume that 
$\eta_k \leq \min \left( \frac{1}{3L_Q}, \frac{1}{15 \gamma_k n} \right)$ for all $k \geq 1$. Then, 
\begin{multline*}
   \E\left[F(x_k)-F^\star+ \frac{\gamma_k}{2}\|v_k-x^\star\|^2\right] \\ \leq \Gamma_k \left(F(x_0)-F^\star + \frac{\gamma_0}{2}\|x_0-x^\star\|^2 + \frac{3\rho_Q\tilde{\sigma}^2 }{n}\sum_{t=1}^k \frac{\eta_t}{\Gamma_t} \right). %\label{eq:acc_rate_svrg}
\end{multline*}
\end{theorem}
We may now derive convergence rates of our accelerated SVRG algorithm under various settings.  
\begin{corollary}[\bf Accelerated prox SVRG with constant $\eta$]~\label{corollary:accsvrg_constant}\newline
With $\eta_k =  \min \left( \frac{1}{3L_Q}, \frac{1}{15 \mu n} \right)$ and $\gamma_0 = \mu$, 
the iterates produced by Algorithm~\ref{alg:acC_svrg} satisfy

~$\bullet$ if $\frac{1}{3L_Q} \leq \frac{1}{15 \mu n}$,
\begin{equation*}
   \E\left[F(x_k)-F^\star\right] \leq \left( 1- \sqrt{\frac{5\mu}{9L_Qn}}\right)^k T_0 + \frac{3\rho_Q \tilde{\sigma}^2}{\sqrt{5 \mu L_Q n}};
\end{equation*}
~$\bullet$ otherwise,
\begin{equation*}
   \E\left[F(x_k)-F^\star\right] \leq \left( 1- \frac{1}{3n} \right)^k T_0 + \frac{3\rho_Q \tilde{\sigma}^2}{5\mu n},
\end{equation*}
with $T_0 = F(x_0)-F^\star + \frac{\mu}{2}\|x_0-x^\star\|^2$.
\end{corollary}
The corollary uses $\Gamma_k\sum_{t=1}^k \eta/\Gamma_t \leq \eta/\delta = \sqrt{3 n \eta/5\mu}$ and thus 
the algorithm converges linearly to an area of radius $3\rho_Q \tilde{\sigma}^2 \sqrt{3 \eta/ 5 \mu n}  = O\left( \rho_Q{\tilde{\sigma}^2}\min \left( \frac{1}{\sqrt{n \mu L_Q}}, \frac{1}{\mu n} \right) \right)$, where as before, $\rho_Q=1$ if $Q$ is uniform. When $\tilde{\sigma}^2=0$, the algorithm achieves the optimal complexity for finite sums~\cite{arjevani2016dimension}. 
Interestingly, we see that here non-uniform sampling may hurt the convergence guarantees in some situations. Whenever $5\mu n > \max_i L_i$, the optimal sampling strategy is indeed the uniform one.
Next, we show how to obtain a converging algorithm.
\begin{corollary}[\bf Accelerated prox SVRG - decreasing~$\eta_k$]\label{corollary:acc_svrg}~\newline
Target an accuracy
$\varepsilon$ smaller than $B=3\rho_Q \tilde{\sigma}^2\sqrt{\eta/\mu}$ with the same step size $\eta$ as in the previous corollary.
First, use such a constant step-size strategy $\eta_k=\eta$ with $\gamma_0=\mu$ within Algorithm~\ref{alg:acC_svrg}, 
 until $\E[F({x}_k)- F^\star] \leq B$.
Then, restart the optimization procedure with decreasing step-sizes $\eta_k
= \min \left(\eta ,\frac{12 n}{5 \mu (k+2)^2}\right)$. The number of gradient evaluations to achieve $\E[F({x}_k)- F^\star] \leq  \varepsilon$ is upper bounded by
$$O\left( \left(n + \sqrt{\frac{{nL_Q}}{\mu}}\right) \log\left(\frac{F(x_0)- F^\star}{\varepsilon}\right)\right) + O\left( \frac{\rho_Q \sigma^2}{\mu \varepsilon}\right).$$
\end{corollary}

\section{Experiments}\label{sec:exp}
Following~\citet{smiso,zheng2018lightweight} we consider logistic regression
with DropOut~\citep{srivastava_dropout:_2014}, which consists of randomly
setting to zero each vector entry with probability $\delta$, leading to the problem
\begin{equation}
    \min_{x \in \Real^p} \frac{1}{n}\sum_{i=1}^n \E_\rho\left[ \log( 1+ e^{-b_i (\rho \circ a_i)^\top x)}\right]  + \frac{\lambda}{2}\|x\|^2, \label{eq:expectation}
\end{equation}
where $\rho$ is a vector in $\{0,1\}^p$ with i.i.d. Bernoulli entries, $\circ$ denotes the elementwise multiplication between two vectors,
the $a_i$'s are vectors in $\Real^p$ and $b_i$ are labels in $\{-1,+1\}$.
Since we normalize the vectors $a_i$, the corresponding functions $f_i$ are $L$-smooth with $L=0.25$.
We consider two DropOut regimes, with $\delta$ in $\{0.01,0.1\}$, representing small and medium perturbations.
The parameter~$\lambda$ acts as a lower bound on~$\mu$ and we consider
$\lambda=1/10n$, which is of the order of the smallest
value that one would try when doing parameter search.
We use three data sets~\textrm{alpha}, \textrm{ckn-cifar}, and~\textrm{gene} from different nature, which
are presented in the appendix, along with other experimental details.

\begin{figure}[t!]
\centering
\includegraphics[width=0.495\linewidth]{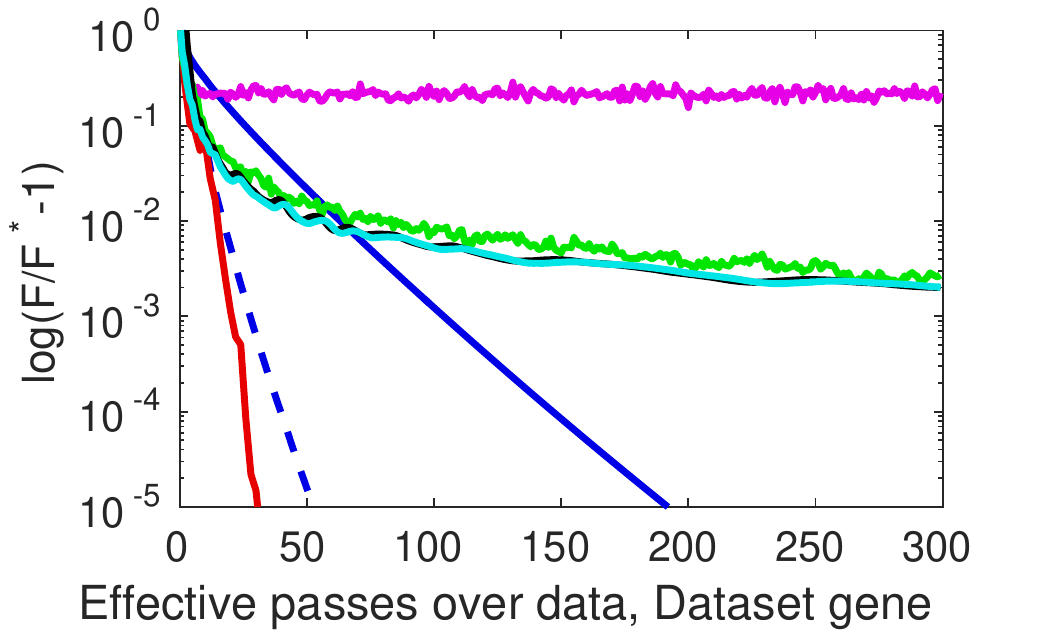} 
\includegraphics[width=0.495\linewidth]{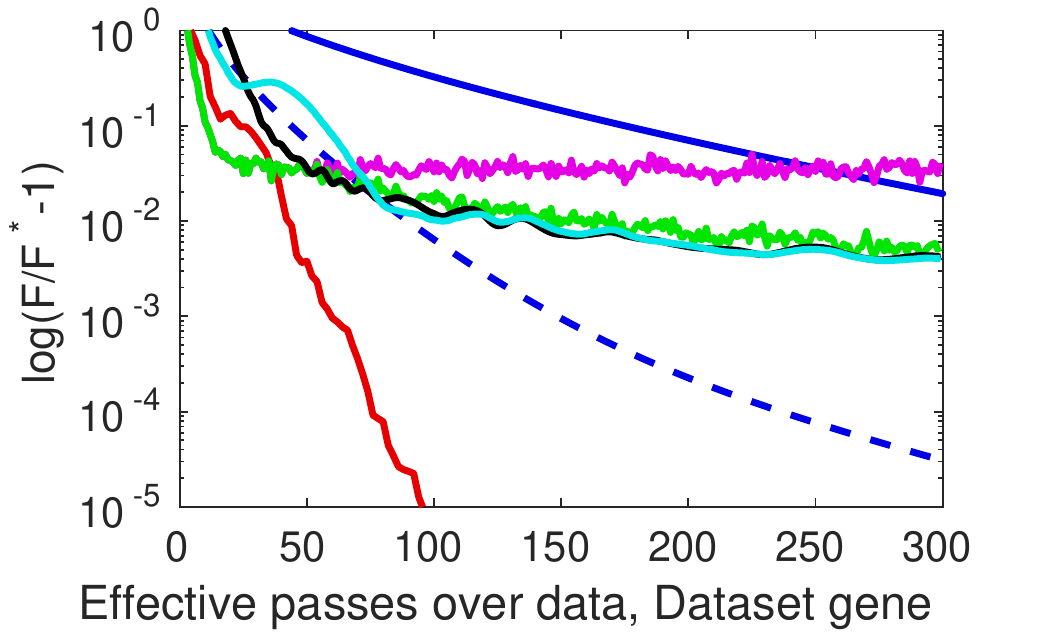}  \\
\includegraphics[width=0.495\linewidth]{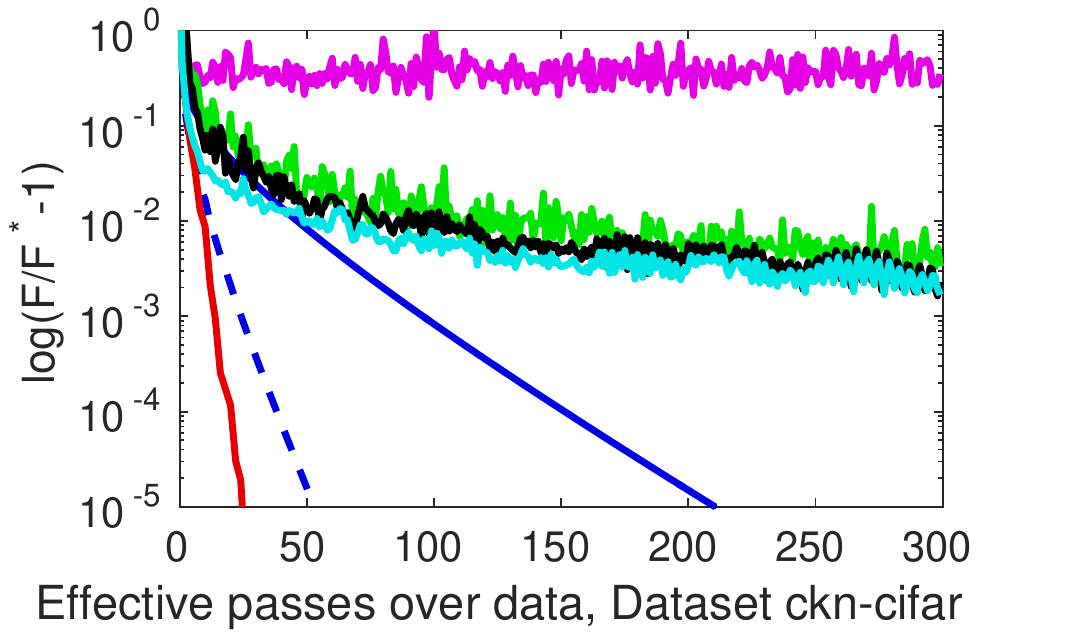} 
\includegraphics[width=0.495\linewidth]{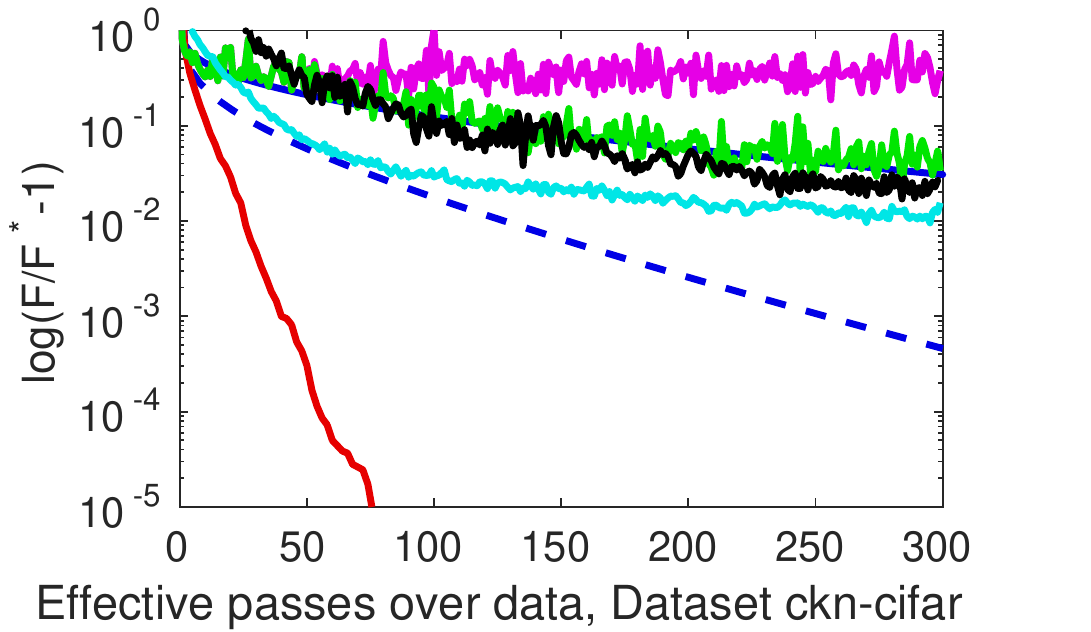}  \\
\includegraphics[width=0.495\linewidth]{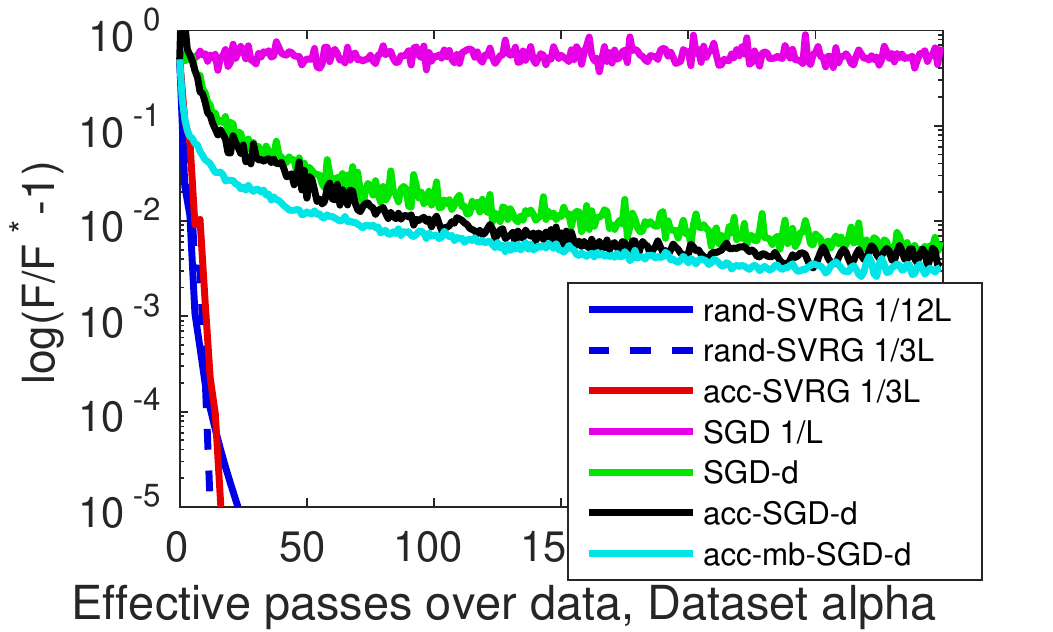} 
\includegraphics[width=0.495\linewidth]{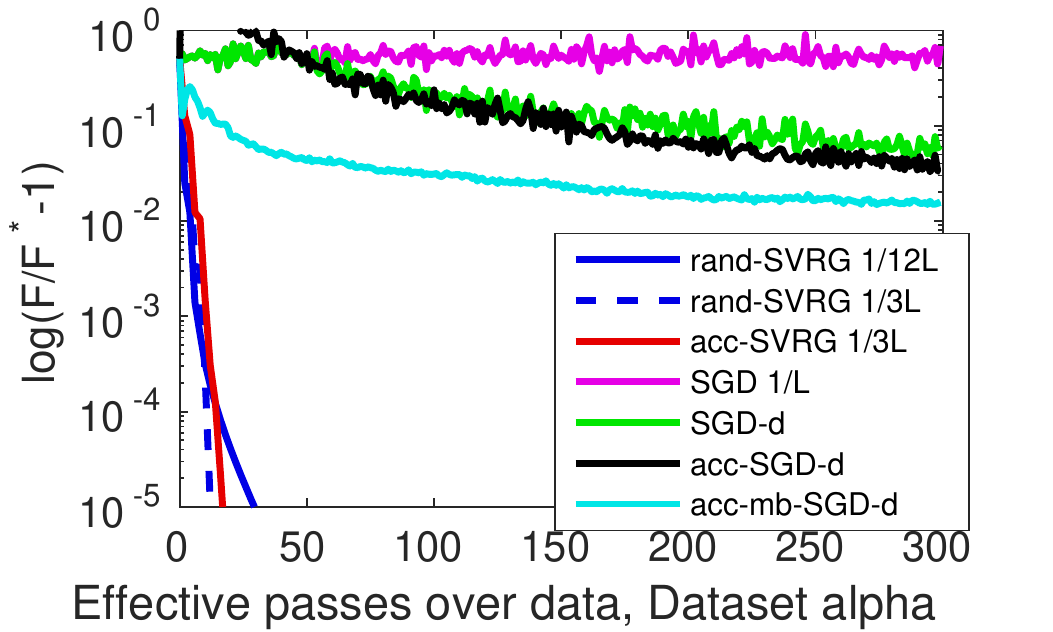} 
\caption{Objective function value on a logarithmic scale with $\lambda=1/10n$ (left) and $\lambda=1/100n$ (right), with no DropOut.}\label{fig:nodropout}
\end{figure}

We consider various methods such as \textrm{SGD}, \textrm{rand-SVRG}, \textrm{acc-SGD} (accelerated SGD), and \textrm{acc-SVRG} (accelerated SVRG).
We use them always with their theoretical step size, except \textrm{rand-SVRG}, which we evaluate with $\eta=1/3L$ in order to obtain a fair comparison with \textrm{acc-SVRG}. When using the decreasing step size strategy, we add the suffix \textrm{-d} to the method's name, and we consider also a minibatch variant of~\textrm{acc-SGD}, denoted by~\textrm{acc-mb-SGD} with minibatch size $b=\sqrt{L/\mu}$.
We also use the initial step size $1/3L$ for \textrm{rand-SVRG-d} since it performs better in practice.
The methods do not use averaging, since it empirically slows down convergence when used from the first iteration;  knowing when to start averaging is indeed not easy and requires heuristics which we do not evaluate here.

\vs
\paragraph{Experiments without perturbation (Figure~\ref{fig:nodropout}).} ~\\
In such a regime, we obtain the following conclusions: \\
~$\bullet$
Acceleration for SVRG is effective on \textrm{gene}
and~\textrm{ckn-cifar} except on~\textrm{alpha}, where all SVRG-like methods
perform already well. This may be due to hidden strong convexity leading to a regime where the complexity is $O(n \log(1/\varepsilon))$, which is independent of the condition number.\\
~$\bullet$ Acceleration is more effective when the problem is badly conditioned---that is, when $\lambda=1/100n$.\\
~$\bullet$ \textrm{acc-mb-SGD-d} performs best among SGD methods and is competitive with \textrm{rand-SVRG} in the low precision regime.

\vs
\paragraph{Experiments with perturbations (Figure~\ref{fig:dropout}).}

\begin{figure}[t!]
\centering
\includegraphics[width=0.495\linewidth]{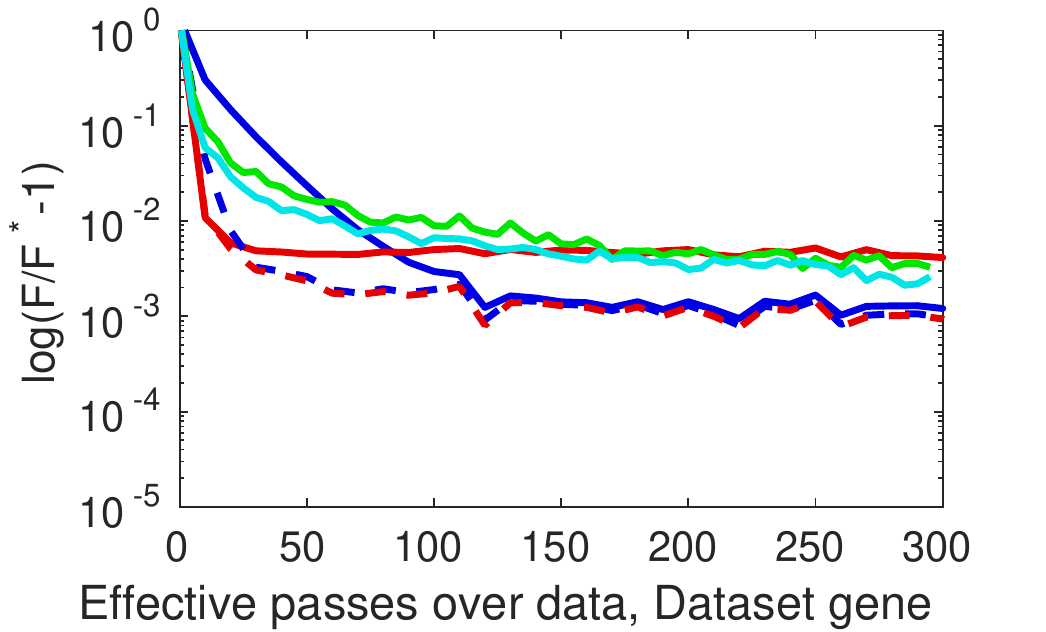} 
\includegraphics[width=0.495\linewidth]{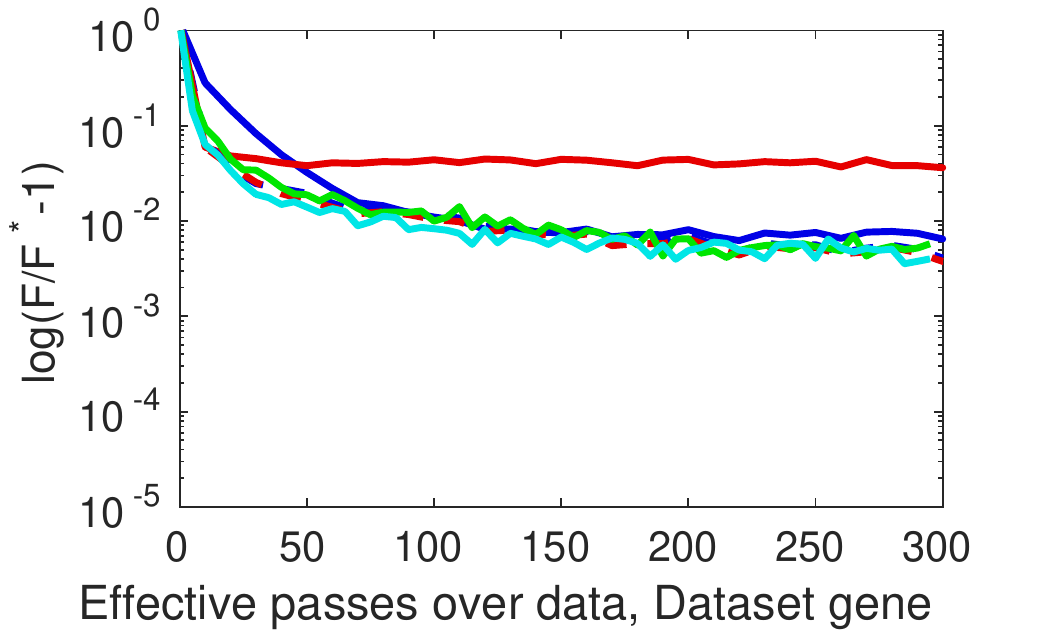}  \\
\includegraphics[width=0.495\linewidth]{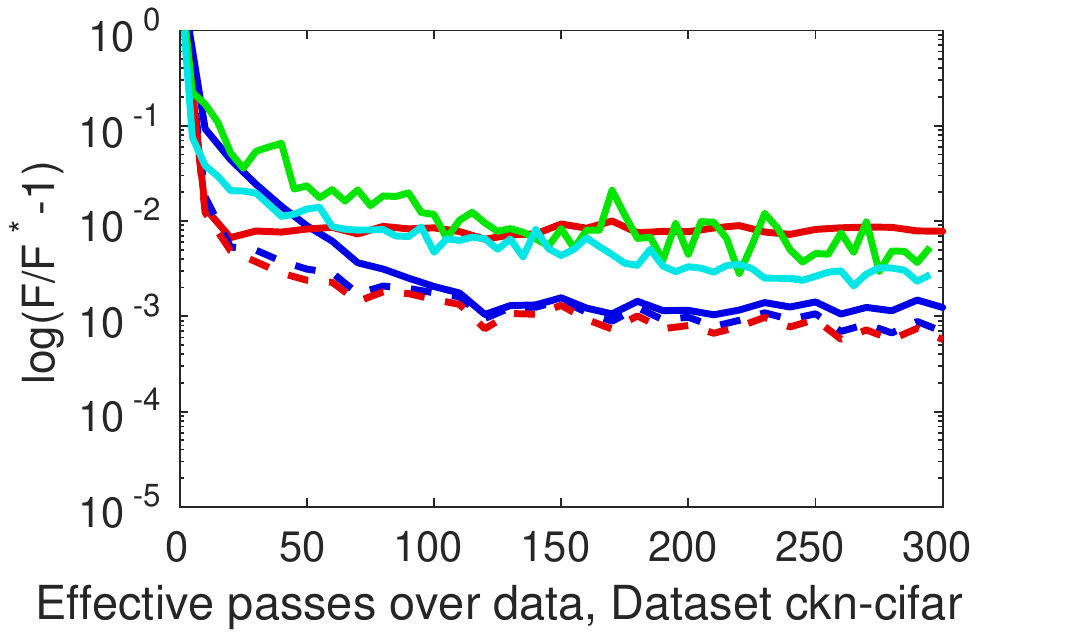} 
\includegraphics[width=0.495\linewidth]{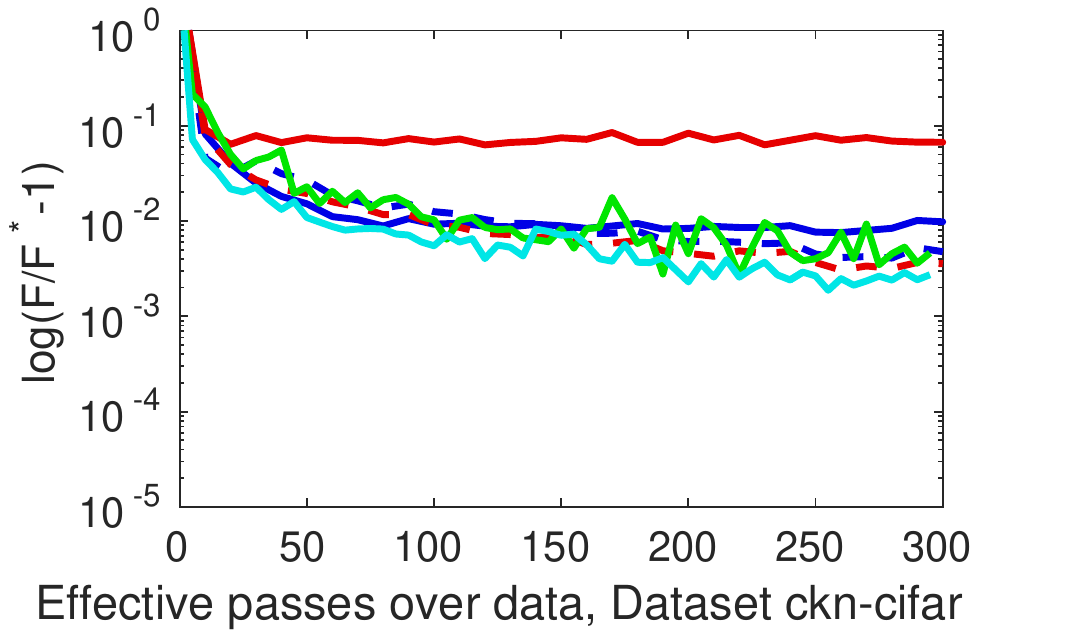}  \\
\includegraphics[width=0.495\linewidth]{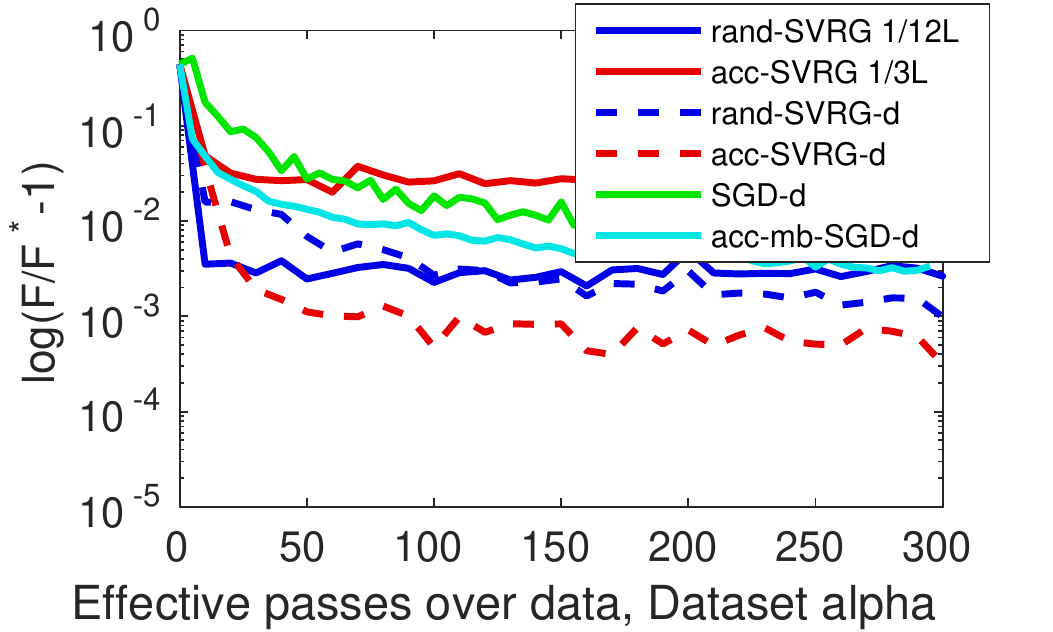} 
\includegraphics[width=0.495\linewidth]{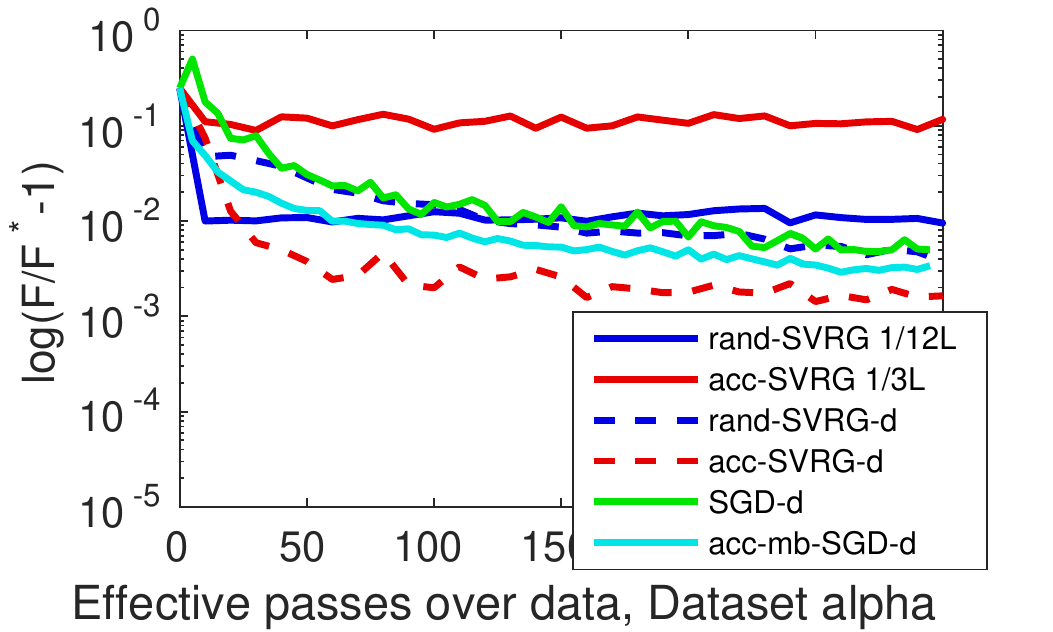} 
\caption{Objective function value on a logarithmic scale with $\lambda=1/10n$, with DropOut $\delta=0.01$ (left) and $\delta=0.1$ (right).}\label{fig:dropout}
\end{figure}

As predicted by theory, approaches with constant step size do not converge.
Therefore, we focus on methods with decreasing step sizes. The conclusions are the following: \\
~$\bullet$ \textrm{acc-mb-SGD-d} with minibatch performs best among SGD approaches and could further benefit from parallelization. \\
~$\bullet$ Acceleration for SVRG is less effective when DropOut is used; the gains are significant on the data set~\textrm{alpha}, and the performance is similar as~\textrm{rand-SVRG} on the two other data sets. Not reported here, acceleration is also more effective with poorly conditioned problems, when $\lambda=1/100n$. \\
~$\bullet$ \textrm{acc-rand-SVRG-d} performs better than SGD approaches in the low perturbation regime $\delta=0.01$ and only on the \textrm{alpha} data set when $\delta=0.1$. Otherwise, the methods perform similarly, making \textrm{acc-rand-SVRG-d} safe to use.

\section*{Acknowledgements}
This work was supported by the ERC grant number 714381 (SOLARIS project). The authors would like to thank Anatoli Juditksy for interesting comments, and the anonymous reviewers who helped improving the manuscript.

\bibliography{../bib}
\bibliographystyle{icml2019}

\clearpage

\appendix

%\aaaaaa

\onecolumn

\vs
\begin{center}
   {\Large Supplementary Material of \\ ``Estimate Sequences for Variance-Reduced Stochastic Composite Optimization''}
\end{center}

\section{Making SAGA Robust to Stochastic Perturbations}\label{appendix:saga}
 \begin{algorithm}[h!]
 \caption{Iteration~(\ref{eq:opt1}) with SAGA estimator}\label{alg:saga}
 \begin{algorithmic}[1]
  \STATE {\bfseries Input:} $x_0$ in $\Real^p$ (initial point); $K$ (number of iterations); $(\eta_k)_{k \geq 0}$ (step sizes); $\beta \in [0,\mu]$;  if averaging, $\gamma_0 \geq \mu$.
  \STATE {\bfseries Initialization:} $z_0^i = \tildenabla f_i(x_0) - \beta x_0$ for all $i=1,\ldots,n$ and $\bar{z}_0 = \frac{1}{n}\sum_{i=1}^n {z_0^i}$.
 \FOR{$k=1,\ldots,K$}
     \STATE Sample $i_k$ according to the distribution $Q=\{q_1,\ldots,q_n\}$;
      \STATE Compute the gradient estimator, possibly corrupted by random perturbations:
         \begin{displaymath}
             g_k = \frac{1}{q_{i_k} n}\left(\tildenabla f_{i_k} (x_\kmone) - \beta x_\kmone - z^{i_k}_\kmone \right) + \bar{z}_\kmone + \beta x_\kmone;
         \end{displaymath}
      \STATE Obtain the new iterate
      $$x_{k} \leftarrow \text{Prox}_{\eta_k\psi}\left[ x_\kmone - \eta_k g_k\right];$$
      \STATE Draw $j_k$ from the uniform distribution in $\{1,\ldots,n\}$;
      \STATE Update the auxiliary variables
         \begin{displaymath}
            z^{j_k}_k  = \tildenabla f_{j_k} (x_k) - \beta x_k ~~~\text{and}~~~ z^j_k  = z^j_\kmone~~~\text{for all}~~~~j \neq j_k;
          \end{displaymath}
     \STATE Update the average variable $\bar{z}_k = \bar{z}_\kmone + \frac{1}{n}(z^{j_k}_k - z^{j_k}_\kmone)$.
      \STATE {\bfseries Optional}: Use the same averaging strategy as in Algorithm~\ref{alg:svrg}.
 \ENDFOR
 \STATE {\bfseries Output:} $x_k$ or $\hat{x}_k$ (if averaging).
 \end{algorithmic}
 \end{algorithm}

\section{Details about the Experimental Setup}
We consider three datasets with various number of points~$n$ and dimension~$p$, coming from different scientific fields:
\begin{itemize}
\vs
\vs
\item \textrm{alpha} is from  the  Pascal  Large  Scale Learning Challenge
website\footnote{\url{http://largescale.ml.tu-berlin.de/}} and contains $n=250\,000$
with $p=500$.
\item \textrm{gene} consists of gene expression data and the binary labels $b_i$ characterize two different types of breast cancer. This is a small dataset with $n=295$ and $p=8\,141$.
\item \textrm{ckn-cifar} is an image classification task where each image from the CIFAR-10 dataset\footnote{\url{https://www.cs.toronto.edu/~kriz/cifar.html}} 
is represented by using a two-layer unsupervised convolutional neural network~\citep{mairal2016end}. Since CIFAR-10 originally contains 10 different classes, we consider the binary classification task consisting of predicting the class 1 vs. other classes. The dataset contains $n=50\,000$ images and the dimension of the representation is $p=9\,216$. 
\end{itemize}
\vs
For simplicity, we normalize the features of all datasets and thus we use a uniform sampling strategy~$Q$ in all algorithms.
Then, we consider several methods with their theoretical step sizes, described in Table~\ref{table:algs}.
Note that we also evaluate the strategy \textrm{random-SVRG} with step size $1/3L$, even though our analysis requires $1/12L$, in order to get a fair comparison with the accelerated SVRG method.
In all figures, we consider that $n$ iterations of SVRG
count as $2$ effective passes over the data since it appears empirically to be a good
proxy of the computational time.
Indeed, (i) if one is allowed to store all variables~$z_i^k$, then $n$ iterations indeed correspond to two passes over the data; (ii) the gradients $\tildenabla f_i(x_\kmone)- \tildenabla f_i(\tilde{x}_\kmone)$ access the same training point which reduces the data access overhead; (iii) computing the full gradient $\bar{z}_k$ can be done in practice in a much more efficient manner than computing individually the $n$ gradients $\tildenabla f_i(x_k)$, either through parallelization or by using more efficient routines (\eg, BLAS2).
Each experiment is conducted five times and we always report the average of the five experiments in each figure.

To evaluate the quality of a solution, when $\tilde{\sigma}^2=0$, we can check
that the value $F^\star$ we consider is optimal by computing a duality gap
using Fenchel duality. In the stochastic case when $\tilde{\sigma}^2 \neq 0$, 
we evaluate the loss function every $5$ data passes and we estimate the expectation~(\ref{eq:expectation}) by drawing $5$ random
perturbations per data point, resulting in $5n$ samples. The optimal value $F^\star$ is estimated by letting the methods
run for $1000$ epochs and selecting the best point found as a proxy of $F^\star$.

\begin{table}[hbtp!]
\definecolor{alizarin}{rgb}{0.82, 0.1, 0.26}
\centering
\begin{tabular}{|c|c|c|c|c|}
\hline
Algorithm & step size $\eta_k$ & Theory & Complexity $O(.)$ & Bias $O(.)$ \\
\hline
\textrm{SGD} & $\frac{1}{L}$ &  Cor.~\ref{corollary:sgd_constant}  & $\frac{L}{\mu}\log\left(\frac{C_0}{\varepsilon}  \right)$ & $\frac{\sigma^2}{L}$ \\
\hline
\textrm{SGD-d} & $\min
\left(\frac{1}{L},\frac{2}{\mu (k+2)}\right)$ &  Cor.~\ref{corollary:sgd}  & $ \frac{L}{\mu}\log\left(\frac{C_0}{\varepsilon}\right) + \frac{\sigma^2}{\mu \varepsilon}  $ & 0 \\
\hline
\textrm{acc-SGD} & $\frac{1}{L}$ &  Cor.~\ref{corollary:acc_sgd2a}  & $ \sqrt{\frac{L}{\mu}}\log\left(\frac{C_0}{\varepsilon}\right)$ & $\frac{\sigma^2}{\sqrt{\mu L}}$ \\
\hline
\textrm{acc-SGD-d} & $\min
\left(\frac{1}{L},\frac{4}{\mu (k+2)^2}\right)$ &  Cor.~\ref{corollary:acc_sgd2}  & $ \sqrt{\frac{L}{\mu}}\log\left(\frac{C_0}{\varepsilon}\right) + \frac{\sigma^2}{\mu \varepsilon}  $ & 0 \\
% \hline
% \textrm{acc-SGD-d}$^\ast$ & $\min
%\left(\frac{\mu^{1/2}}{L^{3/2}},\frac{4}{\mu (k+2)^2}\right)$ &  Cor.~\ref{corollary:acc_sgd2b}  & $ \left({\frac{L}{\mu}}\right)^{3/4}\log\left(\frac{C_0}{\varepsilon}\right) + \frac{\sigma^2}{\mu \varepsilon}  $ & 0 \\
\hline
\textrm{acc-mb-SGD-d} & $\min
\left(\frac{1}{L},\frac{4}{\mu (k+2)^2}\right)$ &  Cor.~\ref{corollary:acc_sgd2}  & $ \frac{L}{\mu}\log\left(\frac{C_0}{\varepsilon}\right) + \frac{\sigma^2}{\mu \varepsilon}  $ & 0 \\
\hline
\textrm{rand-SVRG} & $\frac{1}{12 L}$ &  Cor.~\ref{corollary:svrg0}  & $ \left(n + \frac{L}{\mu}\right)\log\left(\frac{C_0}{\varepsilon}\right) $ & $\frac{{\color{alizarin} \tilde{\sigma}^2}}{L}$ \\
\hline
\textrm{rand-SVRG-d} & $\min\left(\frac{1}{12L_Q}, \frac{1}{5\mu n}, \frac{2}{\mu(k+2)}\right)$ &  Cor.~\ref{corollary:svrg2}  & $ \left(n + \frac{L}{\mu}\right)\log\left(\frac{C_0}{\varepsilon}\right) + \frac{{\color{alizarin}\tilde{\sigma}^2}}{\mu \varepsilon}  $ & 0 \\
\hline
\textrm{acc-SVRG} & $\min \left( \frac{1}{3L_Q}, \frac{1}{15 \mu n} \right)$ &  Cor.~\ref{corollary:accsvrg_constant}  & $ \left(n + \sqrt{\frac{nL}{\mu}}\right)\log\left(\frac{C_0}{\varepsilon}\right) $ & $\frac{{\color{alizarin}\tilde{\sigma}^2}}{\sqrt{n \mu L}+n \mu}$ \\
\hline
\textrm{acc-SVRG-d} & $\min \left( \frac{1}{3L_Q}, \frac{1}{15 \mu n} ,\frac{12 n}{5 \mu (k+2)^2}\right)$ &  Cor.~\ref{corollary:acc_svrg}  & $ \left(n + \sqrt{\frac{nL}{\mu}}\right)\log\left(\frac{C_0}{\varepsilon}\right)  + \frac{{\color{alizarin}\tilde{\sigma}^2}}{\mu \varepsilon}$ & 0 \\
\hline
\end{tabular}
\caption{List of algorithms used in the experiments, along with the step size used and the pointer to the corresponding convergence guarantees, with $C_0=F(x_0)-F^\star$. In the experiments, we also use the method \textrm{rand-SVRG} with step size $\eta=1/3L$. 
The approach \textrm{acc-mb-SGD-d} uses minibatches of size $\lceil\sqrt{L/\mu}\rceil$ and could thus easily be parallelized.
Note that we potentially have ${\color{alizarin}\tilde{\sigma}} \ll \sigma$.}\label{table:algs}
\end{table}

\section{Useful Mathematical Results}
\subsection{Simple Results about Convexity and Smoothness}
The next three lemmas are classical upper and lower bounds for smooth or strongly convex functions~\cite{nesterov}.
\begin{appendixlemma}[\bfseries Quadratic upper bound for $L$-smooth functions]~\label{lemma:upper}\newline
Let $f: \Real^p \to \Real$ be $L$-smooth. Then, for all $x, x'$ in $\Real^p$,
   \begin{equation*}
          |f(x') - f(x) - \nabla f(x)^\top (x'-x)| \leq \frac{L}{2}\|x-x'\|^2.\label{eq:lipschitz}
    \end{equation*}
\end{appendixlemma}
\begin{appendixlemma}[\bfseries Lower bound for strongly convex
   functions]~\label{lemma:lower}\newline
Let $f: \Real^p \to \Real$ be a $\mu$-strongly convex function. Let $z$ be in $\partial f(x)$ for some $x$ in $\Real^p$. Then, the following inequality holds for all $x'$ in $\Real^p$:
\begin{displaymath}
    f(x') \geq f(x) + z^\top(x'-x) + \frac{\mu}{2}\|x-x'\|^2.
\end{displaymath}
\end{appendixlemma}
\begin{appendixlemma}[\bfseries Second-order growth property]~\label{lemma:second}\newline
 Let $f: \Real^p \to \Real$ be a $\mu$-strongly convex function and $\Xcal \subseteq \Real^p$ be a convex set.
  Let $x^\star$ be the minimizer of $f$ on~$\Xcal$. Then, the following condition holds for all $x$ in $\Xcal$:
   \begin{displaymath}
        f(x) \geq f(x^\star) + \frac{\mu}{2}\|x-x^\star\|^2.
         \end{displaymath}
\end{appendixlemma}
% \begin{appendixlemma}[\bfseries Useful inequality for smooth and convex functions]~\label{lemma:useful}\newline
% Consider an $L$-smooth $\mu$-strongly convex function~$f$ defined on $\Real^p$ and a parameter $\beta$ in $[0,\mu]$. Then, for all $x,y$ in $\Real^p$, 
% \begin{displaymath}
%     \| \nabla f(x) - \nabla f(y) - \beta(x-y)\|^2 \leq 2L (f(x)-f(y)-\nabla f(y)^\top(x-y)).
% \end{displaymath}
% \end{appendixlemma}
% \begin{proof}
% Let us define the function $\phi(x) = f(x) - \frac{\beta}{2}\|x\|^2$, which is $(\mu-\beta)$-strongly convex. It is then easy to show that $\phi$ is $(L-\beta)$-smooth, according to~Theorem 2.1.5 in~\cite{nesterov}: indeed, for all $x, y$ in~$\Real^p$,
% \begin{displaymath}
% \begin{split}
%     \phi(x) = f(x) -  \frac{\beta}{2}\|x\|^2 & \leq f(y) + \nabla f(y)^\top(x-y) + \frac{{L}}{2}\|x-y\|^2  -  \frac{\beta}{2}\|x\|^2 \\
%     & = \phi(y) + \nabla \phi(y)^\top (x-y) + \frac{{L}-\beta}{2}\|x-y\|^2,
% \end{split}
% \end{displaymath}
% and again according to Theorem 2.1.5 of~\cite{nesterov}, 
% \begin{displaymath}
% \begin{split}
%     \left\| \nabla \phi(x) - \nabla \phi(y)\right \|^2 & \leq 2{L}( \phi(x) - \phi(y) - \nabla \phi(y)^\top(x -y)) \\
%     & = 2{L}\left( f(x) - f(y) - \nabla f(y)^\top(x -y) - \frac{\beta}{2}\| x - y\|^2 \right)  \\
%     & \leq 2{L}\left( f(x) - f(y) - \nabla f(y)^\top(x -y)\right).
% \end{split}
% \end{displaymath}
% 
% \end{proof}

\subsection{Useful Results to Select Step Sizes}\label{appendix:step}
In this section, we present basic mathematical results regarding the choice of step sizes. The proof of the first two lemmas is trivial by induction.
\begin{appendixlemma}[Relation between $(\delta_k)_{k \geq 0}$ and $(\Gamma_k)_{k \geq 0}$]\label{lemma:step}
Consider the following scenarios for $\delta_k$ and $\Gamma_k=\prod_{t=1}^k(1-\delta_t)$:
\begin{itemize}
   \item $\delta_k=\delta$ (constant). Then $\Gamma_k= (1-\delta)^k$. 
      \item $\delta_k=2/(k+2)$. Then, $\Gamma_k = \frac{2}{(k+1)(k+2)}$. 
      \item $\delta_k = \min( 2/(k+2), \delta)$. Then,
      $$ \Gamma_k = \left\{ \begin{array}{ll} 
      (1-\delta)^k & ~\text{if}~k < k_0 ~~~\text{with}~~~ k_0 = \left\lceil \frac{2}\delta - 2 \right\rceil \\
         \Gamma_{k_0-1} \frac{k_0(k_0+1)}{(k+1)(k+2)} & ~\text{otherwise}.
      \end{array}
      \right.$$
\end{itemize}
\end{appendixlemma}
%Then, the following relation always holds (the proof is trivial by induction);
\begin{appendixlemma}[Simple relation]\label{lemma:simple}
Consider a sequence of weights $(\delta_k)_{k \geq 0}$ in $(0,1)$. Then,
\begin{equation}
   \sum_{t=1}^k \frac{\delta_t}{\Gamma_t} + 1 = \frac{1}{\Gamma_k} \qquad
\text{where} \qquad \Gamma_t \defin \prod_{i=1}^t (1-\delta_i).\label{eq:av}
\end{equation}
\end{appendixlemma}
\begin{appendixlemma}[Convergence rate of $\Gamma_k$]\label{eq:rate_gamma}
   Consider the same quantities defined in the previous lemma and consider the sequence $\gamma_k = (1-\delta_k)\gamma_\kmone + \delta_k \mu = \Gamma_k\gamma_0 + (1-\Gamma_k) \mu$ with $\gamma_0 \geq \mu$, and assume the relation $\delta_k=\gamma_k \eta$.
Then, for all $k \geq 0$,
\begin{equation}
     \Gamma_k \leq \min \left( \left( 1-  \mu \eta\right)^k ,  \frac{1}{1 + {\gamma_0 \eta k}} \right). \label{eq:rate_gamma2}
\end{equation}
Besides, 
\begin{itemize}
  \item when $\gamma_0=\mu$, then $\Gamma_k = (1-\mu\eta)^k$.
  \item when $\mu=0$, $\Gamma_k = \frac{1}{1 + {\gamma_0 \eta k}}$.
\end{itemize}
\end{appendixlemma}
\begin{proof}
    First, we have for all $k$, $\gamma_k \geq \mu$ such that $\delta_k \geq \eta \mu$, which leads then to $\Gamma_k \leq \left(1- \eta{\mu}\right)^k$. Besides, 
    $\gamma_k \geq \Gamma_k \gamma_0$ and thus $\Gamma_k = (1-\delta_k)\Gamma_\kmone \leq\left(1- {\Gamma_k \gamma_0}\eta\right)\Gamma_\kmone$.
    Then, 
    $ \frac{1}{\Gamma_k} \left(1 - \Gamma_k{\gamma_0}\eta \right)  \geq  \frac{1}{\Gamma_\kmone},$ and 
    $$\frac{1}{\Gamma_k} \geq \frac{1}{\Gamma_\kmone} + {\gamma_0}\eta \geq 1 + \gamma_0 \eta k,$$ which is sufficient to obtain~(\ref{eq:rate_gamma2}).
    Then, the fact that $\gamma_0=\mu$ leads to $\Gamma_k = (1-\mu\eta)^k$ is trivial, and the fact that $\mu=0$ yields $\Gamma_k = \frac{1}{1 + {\gamma_0 \eta k}}$ can be shown by induction.
    Indeed, the relation is true for $\Gamma_0$ and then, assuming the relation is true for $k-1$, we have for $k \geq 1$,
    \begin{displaymath}
       \Gamma_k = (1-\delta_k)\Gamma_\kmone = (1-\eta \gamma_k) \Gamma_\kmone =(1-\eta \gamma_0 \Gamma_k) \Gamma_\kmone \geq \left(1-{\eta \gamma_0}\Gamma_k\right) \frac{1}{1+\gamma_0\eta (\kmone)}, 
    \end{displaymath}
    which leads to $\Gamma_k = \frac{1}{1 + {\gamma_0 \eta k}}$.
\end{proof}
\begin{appendixlemma}[Accelerated convergence rate of $\Gamma_k$]\label{eq:acc_rate_gamma}
   Consider the same quantities defined in Lemma~\ref{lemma:simple} and consider the sequence $\gamma_k = (1-\delta_k)\gamma_\kmone + \delta_k \mu = \Gamma_k\gamma_0 + (1-\Gamma_k) \mu$ with $\gamma_0 \geq \mu$, and assume the relation $\delta_k=\sqrt{\gamma_k \eta}$.
Then, for all $k \geq 0$,
\begin{equation*}
   \Gamma_k \leq \min \left( \left( 1- \sqrt{\mu \eta}\right)^k ,  \frac{4}{(2 + {\sqrt{\gamma_0 \eta} k})^2} \right). %\label{eq:rate_gamma2}
\end{equation*}
Besides, 
    when $\gamma_0=\mu$, then $\Gamma_k = (1-\sqrt{\mu\eta})^k$.
\end{appendixlemma}
\begin{proof}
   see Lemma 2.2.4 of~\cite{nesterov}.
\end{proof}

\subsection{Averaging Strategy}
Next, we show a generic convergence result and an appropriate averaging strategy given a recursive relation between quantities acting as Lyapunov function.
\begin{appendixlemma}[Averaging strategy]\label{lemma:averaging}
  Assume that an algorithm generates a sequence $(x_k)_{k \geq 0}$ for minimizing a convex function~$F$, and that there exist sequences $(T_k)_{k \geq 0}$, $(\delta_k)_{k \geq 1}$ in~$(0,1)$, $(\beta_k)_{k \geq 1}$ and a scalar $\alpha >0$ such that  for all $k \geq 1$,
  \begin{equation}
      \frac{\delta_k}{\alpha} \E[ F(x_k) - F^\star]  + T_k \leq (1-\delta_k) T_{\kmone} + \beta_{k}, \label{eq:av1}
  \end{equation}
  where the expectation is taken with respect to any random parameter used by the algorithm.
  Then, we consider two cases:
  \paragraph{No averaging.}
  \begin{equation}
     \E[ F(x_k) - F^\star]  + \frac{\alpha}{\delta_k} T_k \leq  \frac{\alpha \Gamma_k}{\delta_k}\left(T_{0} + \sum_{t=1}^k \frac{\beta_{t}}{\Gamma_t}\right)~~~\text{where}~~~\Gamma_k \defin \prod_{t=1}^k (1-\delta_t). \label{eq:av2}
  \end{equation}
  \paragraph{Averaging.}
  By defining the averaging sequence $(\hat{x}_k)_{k \geq 0}$, 
\begin{displaymath}
\hat{x}_k = \Gamma_k \left( x_0 + \sum_{t=1}^k \frac{\delta_t}{\Gamma_t}x_t \right) = (1-\delta_k) \hat{x}_{\kmone} + \delta_k x_k~~~\text{(for $k \geq 1$)}, 
\end{displaymath}
then, 
\begin{equation}
    \E[ F(\hat{x}_k) - F^\star] + \alpha {T_k} \leq \Gamma_k \left(\alpha T_0 + \E[F(x_0)-F^\star] + \alpha \sum_{t=1}^k \frac{\beta_{t}}{\Gamma_t}\right). \label{eq:av3}
\end{equation}
\end{appendixlemma}
\begin{proof}
Given  that $T_k \leq (1-\delta_k) T_{\kmone} + \beta_{k}$, we obtain~(\ref{eq:av2}) by simply unrolling the recursion.
To analyze the effect of the averaging strategies, divide now~(\ref{eq:av1})
by~$\Gamma_k$:
  \begin{displaymath}
      \frac{\delta_k}{\alpha \Gamma_k} \E[ F(x_k) - F^\star]  + \frac{T_k}{\Gamma_k} \leq \frac{T_{\kmone}}{\Gamma_{\kmone}} + \frac{\beta_{k}}{\Gamma_k}.
  \end{displaymath}
Sum from $t=1$ to~$k$ and notice that we have a telescopic sum:
  \begin{equation*}
    \frac{1}{\alpha} \sum_{t=1}^k \frac{\delta_t}{\Gamma_t} \E[ F(x_t) - F^\star]  + \frac{T_k}{\Gamma_k} \leq T_0 + \sum_{t=1}^k \frac{\beta_{t}}{\Gamma_t}. %\label{eq:aux_averaging}
  \end{equation*}
Then, add $(1/\alpha)\E[F(x_0)-F^\star]$ on both sides and multiply by $\alpha \Gamma_k$:
  \begin{displaymath}
    \sum_{t=1}^k \frac{\delta_t\Gamma_k}{\Gamma_t} \E[ F(x_t) - F^\star]  + \Gamma_k \E[F(x_0)-F^\star] + \alpha {T_k} \leq \Gamma_k \left(\alpha T_0 + \E[F(x_0)-F^\star] + \alpha \sum_{t=1}^k \frac{\beta_{t}}{\Gamma_t}\right).
\end{displaymath}
By exploiting the relation~(\ref{eq:av}), we may then use Jensen's inequality and we obtain~(\ref{eq:av3}).
%For the second averaging strategy, we simply multiply~(\ref{eq:aux_averaging}) by $\alpha \Gamma_k/(1-\Gamma_k)$ and use Jensen's inequality.
\end{proof}

\section{Proofs of the Main Results}\label{appendix:applications}

\subsection{Proof of Proposition~\ref{prop:keyprop}}
\begin{proof}
\begin{displaymath}
\begin{split}
    d_k^\star = d_k(x_k) & = (1-\delta_k) d_{\kmone}(x_k) + \delta_k \left( f(x_\kmone) + g_k^\top( x_k-x_\kmone) + \frac{\mu}{2}\|x_k-x_\kmone\|^2 + \psi(x_k) \right) \\
              & \geq (1-\delta_k) d_{\kmone}^\star + \frac{\gamma_k}{2}\|x_k-x_{\kmone}\|^2 + \delta_k \left( f(x_\kmone)  + g_k^\top( x_k-x_\kmone) + \psi(x_k) \right)  \\
              & \geq (1-\delta_k) d_{\kmone}^\star + \delta_k \left( f(x_\kmone)  + g_k^\top( x_k-x_\kmone) + \frac{L}{2}\|x_k-x_\kmone\|^2 + \psi(x_k) \right)  \\
              & \geq (1-\delta_k) d_{\kmone}^\star + \delta_k F(x_k) +  \delta_k (g_k - \nabla f(x_\kmone))^\top (x_k-x_\kmone), \\
\end{split}
\end{displaymath}
where the first inequality comes from Lemma~\ref{lemma:second}---it is in fact an equality when considering Algorithm~(\ref{eq:opt1})---and the second inequality simply uses the assumption $\eta_k \leq 1/L$, which yields $\delta_k = \gamma_k\eta_k \leq \gamma_k/L$. Finally, the last inequality uses a classical upper-bound for $L$-smooth functions presented in Lemma~\ref{lemma:upper}.
Then, after taking expectations, 
\begin{equation*}
\begin{split}
E[d_k^\star]      & \geq (1-\delta_k) \E[d_{\kmone}^\star] + \delta_k \E[F(x_k)] + \delta_k \E[(g_k-\nabla f(x_\kmone))^\top(x_k-x_\kmone)] \\
     & = (1-\delta_k) \E[d_{\kmone}^\star] + \delta_k \E[F(x_k)] + \delta_k \E[(g_k-\nabla f(x_\kmone))^\top x_k]  \\
     & = (1-\delta_k) \E[d_{\kmone}^\star] + \delta_k \E[F(x_k)] + \delta_k \E\left[(g_k-\nabla f(x_\kmone))^\top\left( x_k - w_\kmone \right) \right], \\
    \end{split} %\label{eq:aux}
\end{equation*}
where we have defined the following quantity 
\begin{displaymath}
 w_{\kmone} = \text{Prox}_{\eta_k \psi}\left[ x_\kmone - \eta_k\nabla f(x_{\kmone})\right].
\end{displaymath}
In the previous relations, we have used twice the fact that 
$\E[(g_k-\nabla f(x_\kmone))^\top y | {\Fcal_{\kmone}}]=0$, for all $y$ that is deterministic given $x_{\kmone}$ such as $y=x_{\kmone}$ or $y=w_{\kmone}$.
We may now use the non-expansiveness property of the proximal operator~\citep{moreau1965} to control the quantity $\|x_k-w_\kmone\|$, which gives us
\begin{equation*}
\begin{split}
E[d_k^\star]
     & \geq (1-\delta_k) \E[d_{\kmone}^\star] + \delta_k \E[F(x_k)] - \delta_k \E\left[\|g_k-\nabla f(x_\kmone)\|\|x_k - w_\kmone\| \right] \\
     & \geq (1-\delta_k) \E[d_{\kmone}^\star] + \delta_k \E[F(x_k)] - \delta_k \eta_k \E\left[\|g_k-\nabla f(x_\kmone)\|^2\right] \\
     & = (1-\delta_k) \E[d_{\kmone}^\star] + \delta_k \E[F(x_k)] - \delta_k \eta_k\sigma_k^2. \\
\end{split}
\end{equation*}
This relation can now be combined with~(\ref{eq:est}) when $z=x^\star$, 
and we obtain~(\ref{eq:relation}). 
\end{proof}

\subsection{Proof of Corollary~\ref{corollary:sgd}}
\begin{proof}
Given the linear convergence rate~(\ref{eq:sgd}), the number of iterations to guarantee $\E[F(\hat{x}_k)- F^\star] \leq  2
\sigma^2/L$ with the constant step-size strategy is upper bounded by
$$O\left( \frac{L}{\mu} \log\left(\frac{F(x_0)-
F^\star}{\varepsilon}\right)\right).$$
Then, after restarting the algorithm, we may apply Theorem~\ref{thm:conv} with $\E[F(x_0)-F^\star] \leq 2 \sigma^2/L$.
With $\gamma_0=\mu$, we have $\gamma_k=\mu$ for all $k \geq 0$, and 
the rate of $\Gamma_k$ is given by Lemma~\ref{lemma:step}, which yields for $k \geq k_0 = \left\lceil \frac{2L}{\mu} - 2 \right\rceil$,
\begin{equation*}
\begin{split}
   \E[F(\hat{x}_k)- F^\star]  & \leq \Gamma_k\left( \E\left[F(x_0)-F^\star + \frac{\mu}{2}\|x_0-x^\star\|^2\right] + \sigma^2 \sum_{t=1}^k \frac{\delta_t \eta_t}{\Gamma_t} \right) \\
   & \leq  \Gamma_k \left( \frac{4 \sigma^2}{L} + \frac{\sigma^2}{L} \sum_{t=1}^{k_0-1} \frac{\delta_t}{\Gamma_t} + \sigma^2 \sum_{t=k_0}^{k} \frac{2\delta_t}{\Gamma_t \mu(t+2)}\right) \\
   & =  \frac{k_0(k_0+1)}{(k+1)(k+2)} \left( \Gamma_{k_0-1}\frac{4 \sigma^2}{L} + \frac{\sigma^2}{L} \Gamma_{k_0-1}\sum_{t=1}^{k_0-1} \frac{\delta_t}{\Gamma_t}\right) + \sigma^2 \sum_{t=k_0}^{k} \frac{2\delta_t \Gamma_k}{\Gamma_t \mu(t+2)} \\
   & =  \frac{k_0(k_0+1)}{(k+1)(k+2)} \left( \Gamma_{k_0-1}\frac{4 \sigma^2}{L} + (1-\Gamma_{k_0 -1 })\frac{\sigma^2}{L}\right) + \sigma^2 \sum_{t=k_0}^{k} \frac{2\delta_t \Gamma_k}{\Gamma_t \mu(t+2)} \\
    & \leq  \frac{k_0(k_0+1)}{(k+1)(k+2)} \frac{4 \sigma^2}{L} +  \sigma^2 \frac{1}{(k+1)(k+2)} \left(\sum_{t=k_0+1}^{k} \frac{4 (t+1)(t+2)}{\mu(t+2)^2}\right) \\
    & \leq  \frac{k_0}{(k+1)(k+2)} \frac{8 \sigma^2}{\mu} +   \frac{4\sigma^2}{\mu(k+2)},
   \end{split}
\end{equation*}
where the second inequality uses the fact that $\frac{\mu}{2}\|x_0-x^\star\|^2 \leq F(x_0)-F^\star \leq \frac{2 \sigma^2}{L}$, and then we use Lemmas~\ref{lemma:step} and~\ref{lemma:simple}. The term on the right is of order $O(\sigma^2/\mu k)$ whereas the term on the left becomes of the same order or smaller whenever $k \geq k_0 = O(L/\mu)$. 
This leads to the desired iteration complexity.
\end{proof}

% \subsection{Proof of Corollary~\ref{corollary:sgd2}}
% \begin{proof}
%    By applying Lemma~\ref{lemma:simple} and~\ref{eq:rate_gamma} to Theorem~\ref{thm:conv}, we obtain the relation
%    \begin{displaymath}
%       \E[F(\hat{x}_K)-F^\star] \leq \frac{T_0}{\eta K} + \sigma^2 \eta.   ~~~\text{with}~~~ T_0 =\frac{1}{2}\|x_0-x^\star\|^2,
%    \end{displaymath}
%    by noticing that $\Gamma_K\sum_{t=1}^K \frac{\delta_t}{\Gamma_t} = 1-\Gamma_K$, and that $\Gamma_K/(1-\Gamma_K) = \frac{1}{\eta L K}$.

%\end{proof}

\subsection{Proof of Proposition~\ref{prop:nonu}}

\begin{proof}
The proof borrows a large part of the analysis of~\citet{proxsvrg} for
controlling the variance of the gradient estimate in the SVRG algorithm. 
First, we note that all the gradient estimators we consider may be written as
           \begin{equation*}
            g_k = \frac{1}{q_{i_k} n}\left(\tildenabla f_{i_k} (x_\kmone) - z_{\kmone}^{i_k} \right) + \bar{z}_\kmone. %\label{eq:gk2}
        \end{equation*}
Then, we will write $\tildenabla f_{i_k}(x_\kmone) = \nabla f_{i_k}(x_\kmone) + \zeta_k$, where $\zeta_k$ is a zero-mean variable with variance $\tilde{\sigma}^2$ drawn at iteration $k$,
and $z_k^i = u_k^i + \zeta_k^i$ for all $k,i$, where~$\zeta_k^i$ has zero-mean with variance $\tilde{\sigma}^2$ and was drawn during the previous iterations.
Then, %we proceed similarly but we now exploit the relation $\bar{z}_k = \frac{1}{n}\sum_{i=1}^n z_i^k$ and 
\begin{equation*}
\begin{split}
    \sigma_k^2 & = \E\left\| \frac{1}{q_{i_k} n}(\tildenabla f_{i_k}(x_{\kmone}) - z^{i_k}_{\kmone}) + \bar{z}_{\kmone}  - \nabla f(x_\kmone)\right\|^2 \\
         & = \E\left\| \frac{1}{q_{i_k} n} (\nabla f_{i_k}(x_{\kmone})  - z^{i_k}_{\kmone}) + \bar{z}^{\kmone}  - \nabla f(x_\kmone) \right\|^2  + \E\left[ \frac{1}{(q_{i_k} n)^2} \|\zeta_k\|^2 \right] \\
         & \leq \E\left\| \frac{1}{q_{i_k} n} (\nabla f_{i_k}(x_{\kmone})  - z^{i_k}_{\kmone}) + \bar{z}^{\kmone}  - \nabla f(x_\kmone) \right\|^2  + \rho_Q{\tilde{\sigma}^2} \\
               & \leq \E\left\| \frac{1}{q_{i_k} n}(\nabla f_{i_k}(x_{\kmone})  - z^{i_k}_{\kmone})\right\|^2 +  \rho_Q{\tilde{\sigma}^2} \\
               & = \frac{1}{n}\sum_{i=1}^n\frac{1}{q_i n}\E\left[\|\nabla f_{i}(x_{\kmone})  - z_\kmone^i \|^2\right] + \rho_Q{\tilde{\sigma}^2} \\
               & = \frac{1}{n}\sum_{i=1}^n\frac{1}{q_i n}\E\left[\|\nabla f_{i}(x_{\kmone})  - u_\star^i + u_\star^i - z_\kmone^i \|^2\right] + \rho_Q{\tilde{\sigma}^2} ~~~\text{with}~~~u_i^\star = \nabla f_i(x^\star) \\
               & \leq \frac{2}{n}\sum_{i=1}^n\frac{1}{q_i n}\E\left[\|\nabla f_{i}(x_{\kmone})  - u_\star^i \|^2 \right] +   \frac{2}{n}\sum_{i=1}^n\frac{1}{q_i n}\E\left[\|z_\kmone^i - u_\star^i  \|^2\right] + \rho_Q{\tilde{\sigma}^2} \\
               & \leq \frac{2}{n}\sum_{i=1}^n\frac{1}{q_i n}\E\left[\|\nabla f_{i}(x_{\kmone})\!-\! \nabla f_i(x^\star) )\|^2 \right] \!+\!   \frac{2}{n}\sum_{i=1}^n\frac{1}{q_i n}\E\left[\|u_\kmone^i \!-\! u_\star^i\|^2\right] + {3\rho_Q \tilde{\sigma}^2} \\
               & \leq \frac{4}{n}\sum_{i=1}^n\frac{L_i}{q_i n}\E\left[f_{i}(x_{\kmone})\!-\! f_i(x^\star) \!-\! \nabla f_i(x^\star)^\top (x_\kmone \!-\! x^\star) \right] \!+\!   \frac{2}{n}\sum_{i=1}^n\frac{1}{q_i n}\E\left[\|u_\kmone^i \!-\! u_\star^i\|^2\right] \!+\! {3\rho_Q \tilde{\sigma}^2} \\
               & \leq {4L_Q}\E\left[f(x_{\kmone})- f(x^\star) - \nabla f(x^\star)^\top (x_\kmone - x^\star) \right] +   \frac{2}{n}\sum_{i=1}^n\frac{1}{q_i n}\E\left[\|u_\kmone^i - u_\star^i\|^2\right] + {3\rho_Q \tilde{\sigma}^2}, \\
\end{split} %label{eq:aux2}
\end{equation*}
where the second inequality uses the relation $\E[\|X- \E[X]\|^2] \leq \E[\|X\|^2]$ for all random variable $X$, taking here expectation with respect to the index $i_k \sim Q$ and conditioning on~$\Fcal_\kmone$; the third inequality uses the relation $\|a+b\|^2 \leq 2\|a\|^2 + 2\|b\|^2$; the fifth inequality uses Theorem 2.1.5 of~\cite{nesterov}.
% \paragraph{Merging the results.}
% Then, let us define the functions $\phi_i(x) = f_i(x) - \frac{\beta}{2}\|x\|^2$, which are $(\mu-\beta)$-strongly convex.\footnote{With an abuse of terminology, a $0$-strongly convex function is simply a convex function.} It is then easy to show that $\phi_i$ is $({L_i}-\beta)$-smooth, according to~Theorem 2.1.5 in~\cite{nesterov}: for all $x, y$ in~$\Real^p$,
% \begin{displaymath}
% \begin{split}
%     \phi_i(x) = f_i(x) -  \frac{\beta}{2}\|x\|^2 & \leq f_i(y) + \nabla f_i(y)^\top(x-y) + \frac{{L_i}}{2}\|x-y\|^2  -  \frac{\beta}{2}\|x\|^2 \\
%     & = \phi_i(y) + \nabla \phi_i(y)^\top (x-y) + \frac{{L_i}-\beta}{2}\|x-y\|^2,
% \end{split}
% \end{displaymath}
% and still according to Theorem 2.1.5 of~\cite{nesterov}, 
% \begin{displaymath}
% \begin{split}
%     \left\| \nabla \phi_{i}(x_{\kmone}) - \nabla \phi_{i}(x^\star)\right \|^2 & \leq 2{L_i}( \phi_i(x_\kmone) - \phi_i(x^\star) - \nabla \phi_i(x^\star)^\top(x_\kmone -x^\star)) \\
%     & = 2{L_i}\left( f_i(x_\kmone) - f_i(x^\star) - \nabla f_i(x^\star)^\top(x_\kmone -x^\star) - \frac{\beta}{2}\| x_\kmone - x^\star\|^2 \right)  \\
%     & \leq 2{L_i}\left( f_i(x_\kmone) - f_i(x^\star) - \nabla f_i(x^\star)^\top(x_\kmone -x^\star)\right).
% \end{split}
% \end{displaymath}
% By taking a weighted average from $i=1$ to $n$,
% \begin{displaymath}
% \frac{1}{n}\sum_{i=1}^n \frac{1}{q_i n}\left\| \nabla f_{i}(x_{\kmone}) - \nabla f_{i}(x^\star)  - \beta (x_\kmone - x^\star)\right \|^2 \leq 2L_Q( f(x_\kmone) - f(x^\star) - \nabla f(x^\star)^\top (x_{\kmone}-x^\star)).
% \end{displaymath}

Then, since $x^\star$ minimizes $F$, we have $0 \in \nabla f(x^\star) + \partial \psi(x^\star)$ and thus $-\nabla f(x^\star)$ is a subgradient in $\partial \psi(x^\star)$. By using as well the convexity inequality $\psi(x) \geq \psi(x^\star) - \nabla f(x^\star)^\top(x-x^\star)$, we obtain
\begin{equation*}
f(x_{\kmone})- f(x^\star) - \nabla f(x^\star)^\top (x_\kmone - x^\star) 
 \leq 2L_Q( F(x_\kmone) - F^\star). %\label{eq:reduc_variant}
\end{equation*}
Finally, given the previous relations, we obtain~(\ref{eq:var2}).
\end{proof}

\subsection{Proof of Proposition~\ref{thm:lyapunov}}
\begin{proof}
To make the notation more compact, we call 
\begin{displaymath}
F_k = \E[F(x_k)-F^\star],\qquad D_k = \E[d_k(x^\star)-d_k^\star] \qquad \text{and}~~~~
C_k = \E\left[\frac{1}{n} \sum_{i=1}^n \frac{1}{q_i n} \| u^i_k - u^i_\star\|^2\right].
\end{displaymath}
Then, according to~Proposition~\ref{prop:nonu}, we have
\begin{displaymath}
  \sigma_k^2 \leq 4 L_Q F_\kmone + 2 C_\kmone + {3 \rho_Q \tilde{\sigma}^2},
\end{displaymath}
and according to Proposition~\ref{prop:keyprop},
\begin{equation}
   \delta_k F_k + D_k \leq (1-\delta_k) D_\kmone  + 4 L_Q \eta_k \delta_k F_\kmone + 2 \eta_k \delta_k C_\kmone + {3\rho_Q \eta_k \delta_k \tilde{\sigma}^2}. \label{eq:key}
\end{equation}
Then, we note that both for the SVRG and SAGA, we have,
\begin{displaymath}
   \E[\| u_k^i - u^i_\star\|^2] = \left(1 - \frac{1}{n}\right)\E[\| u^i_\kmone - u^i_\star\|^2] + \frac{1}{n}\E \|\nabla f_i(x_k) - \nabla f_i(x^\star)\|^2.
\end{displaymath}
By taking a weighted average, this yields
\begin{equation*}
\begin{split}
   C_k & \leq \left(1 - \frac{1}{n}\right)C_\kmone + \frac{1}{n^2}\sum_{i=1}^n\frac{1}{q_i n}\E\left[\|\nabla f_{i}(x_{k})- \nabla f_i(x^\star)\|^2 \right] \\
       & \leq \left(1 - \frac{1}{n}\right)C_\kmone + \frac{1}{n^2}\sum_{i=1}^n\frac{2L_i}{q_i n}\E\left[f_i(x_k)- f_i(x^\star)-\nabla f_i(x^\star)^\top(x_k-x^\star) \right] \\
       & \leq \left(1 - \frac{1}{n}\right)C_\kmone + \frac{2 L_Q F_k}{n},
\end{split}
\end{equation*}
where the second inequality comes from Theorem 2.1.5 of~\cite{nesterov} and the last one uses similar arguments as in the proof of Proposition~\ref{prop:nonu}.
Then, we add a quantity $\beta_k C_k$ on both sides of the relation~(\ref{eq:key}) with some $\beta_k > 0$ that we will specify later:
\begin{equation*}
   \left(\delta_k-\beta_k \frac{2L_Q}{n}\right) F_k + D_k + \beta_k C_k  \leq (1-\delta_k) D_\kmone + \left(\beta_k\left(1 - \frac{1}{n}\right)+2 \eta_k \delta_k \right) C_\kmone + 4 L_Q \eta_k \delta_k F_\kmone  + {3\rho_Q \eta_k \delta_k\tilde{\sigma}^2}, 
\end{equation*}
and then choose  $\frac{\beta_k}{n} = \frac{5}{2} \eta_k \delta_k $, which yields
 \begin{equation*}
    \delta_k\left(1-5 L_Q \eta_k\right) F_k + D_k + \beta_k C_k \leq (1-\delta_k) D_\kmone + \beta_k\left(1 - \frac{1}{5n}\right) C_\kmone + 4 L_Q \eta_k \delta_k F_\kmone  + {3 \rho_Q \eta_k \delta_k\tilde{\sigma}^2}. %\label{eq:key}
 \end{equation*}
Remember that $\tau_k = \min\left(\delta_k, \frac{1}{5n}\right)$, notice that the sequences $(\beta_k)_{k \geq 0}, (\eta_k)_{k \geq 0}$ and
$(\delta_k)_{k \geq 0}$ are non-increasing and note that ${4} \leq {5}(1-\frac{1}{5n})$ for all $n \geq 1$. Then, 
 \begin{equation*}
    \delta_k\left(1-10 L_Q \eta_k\right) F_k + \underbrace{5L_Q \eta_k \delta_k +  D_k + \beta_k C_k}_{T_k} \leq (1-\tau_k) \left(D_\kmone + \beta_\kmone C_\kmone + 5 L_Q \eta_\kmone \delta_\kmone F_\kmone\right)  + {3 \rho_Q \eta_k \delta_k\tilde{\sigma}^2},
\end{equation*}
which immediately yields~(\ref{eq:aux2}) with the appropriate definition of $T_k$, and by noting that $(1-10 L_Q\eta_k) \geq \frac{1}{6}$.
\end{proof}

\subsection{Proof of Corollary~\ref{corollary:svrg0}}
\begin{proof}
First, notice that (i) $T_k \geq d_k(x^\star)-d_k^\star \geq \frac{\mu}{2}\|x_k-x^\star\|^2$, that (ii) $\delta_k = \eta_k \gamma_k = \frac{\mu}{12 L_Q}$ and
that $\mu \frac{\tau_k}{\delta_k} = \min\left(\mu, \frac{12L_Q}{5n}\right)$. Then, we apply Theorem~\ref{thm:svrg} and obtain
\begin{displaymath}
\begin{split}
   \E\left[F(\hat{x}_k)-F^\star  + \alpha \|x_k-x^\star\|^2\right] & \leq \Theta_k \left( F(x_0)-F^\star + \frac{6 \tau_k}{\delta_k} T_0 + \frac{18 \rho_Q \tau_k \tilde{\sigma}^2}{\delta_k} \sum_{t=1}^k \frac{\eta_t \delta_t}{\Theta_t}     \right) \\
    & = \Theta_k \left( F(x_0)-F^\star + \frac{6 \tau_k}{\delta_k} T_0 + \frac{3 \rho_Q \tilde{\sigma}^2}{2L_Q} \sum_{t=1}^k \frac{\tau_t}{\Theta_t}     \right) \\
    & \leq \Theta_k \left( F(x_0)-F^\star + \frac{6 \tau_k}{\delta_k} T_0 \right)  + \frac{3 \rho_Q \tilde{\sigma}^2}{2L_Q}. \\
   \end{split}
\end{displaymath}
Then, note that
\begin{displaymath}
\begin{split}
   T_0 & = \frac{5\delta_0}{12}(F(x_0)-F^\star) + \frac{\mu}{2}\|x_0-x^\star\|^2 + \frac{5 \delta_0}{24 L_Q n}\sum_{i=1}^n \frac{1}{q_i n} \| u^i_0 - u^i_\star\|^2 \\
    & \leq \frac{5\delta_0}{12}(F(x_0)-F^\star) + \frac{\mu}{2}\|x_0-x^\star\|^2 + \frac{5 \delta_0}{12}(F(x_0)-F^\star), 
\end{split}
\end{displaymath}
where the inequality comes from  Theorem 2.1.5 of~\cite{nesterov} and the definition of the $u^i_0$'s. Then, we conclude by noting that $5 \tau \leq 1$, and that $\alpha \leq 3 \mu$ and we use Lemma~\ref{lemma:second}.
\end{proof}

\subsection{Proof of Corollary~\ref{corollary:svrg2}}
\begin{proof}
We start by following similar steps as in the proof of Corollary~\ref{corollary:svrg0} to study
the convergence of the first phase with constant step size.
We note that with the choice of $\eta_k$, we have $\delta_k = \tau_k$ for all $k$. Then, we apply Theorem~\ref{thm:svrg} and obtain
\begin{displaymath}
\begin{split}
   \E\left[F(\hat{x}_k)-F^\star  + {3\mu} \|x_k-x^\star\|^2\right] & \leq \Theta_k \left( F(x_0)-F^\star + 6 T_0 + {18 \rho_Q \tilde{\sigma}^2 \eta} \sum_{t=1}^k \frac{\tau_t}{\Theta_t}     \right)\\
    & \leq \Theta_k \left( F(x_0)-F^\star + {6} T_0\right) + {18 \rho_Q \tilde{\sigma}^2 \eta}.  \\
   \end{split}
\end{displaymath}
Then, we use the same upper-bound on $T_0$ as in the proof of
Corollary~\ref{corollary:svrg0}, giving us $6 T_0 \leq 5 \delta_0
(F(x_0)-F^\star) + 3 \mu \|x_0-x^\star\|^2 \leq 7 (F(x_0)-F^\star)$ since
$\delta_0 = \mu \eta \leq 1/5$,
which is sufficient to conclude that
\begin{equation}
   \E\left[F(\hat{x}_k)-F^\star  + 3 \mu \|x_k-x^\star\|^2\right] \leq 8 \Theta_k\left( F(x_0)-F^\star \right)  + {18 \rho_Q \eta\tilde{\sigma}^2}.
    \label{eq:svrg_constant2}
\end{equation}

Then, we restart the procedure. Since the convergence rate~(\ref{eq:svrg_constant2}) applies for the first stage with a constant step size, the number of iterations to ensure the condition $\E[F(\hat{x}_k)-F^\star] \leq 24\eta\rho_Q \tilde{\sigma}^2$ is upper bounded by $K$ with
 \begin{displaymath}
    K = O\left( \left(n + \frac{L_Q}{\mu}\right) \log\left(\frac{F(x_0)-F^\star}{\varepsilon} \right)  \right).
\end{displaymath}
Then, we restart the optimization procedure, assuming from now on that $\E[F(x_0)-F^\star] \leq 24\eta \rho_Q \tilde{\sigma}^2$, with decreasing step sizes 
$\eta_k = \min\left( \frac{2}{\mu(k+2)}, {\eta} \right)$,
Then, since $\delta_k = \mu \eta_k \leq \frac{1}{5n}$, we have that $\tau_k = \delta_k$ for all $k$, and
Theorem~\ref{thm:svrg} gives us---note that here $\Gamma_k=\Theta_k$---
\begin{displaymath}
   \E\left[F(\hat{x}_k)-F^\star\right] \leq \Gamma_k \left( F(x_0)-F^\star + {6}T_0 + {18 \rho_Q \tilde{\sigma}^2} \sum_{t=1}^k \frac{\eta_t \delta_t}{\Gamma_t}     \right)~~~\text{with}~~~ \Gamma_k = \prod_{t=1}^k(1-\delta_t).
\end{displaymath}
Then, as noted in the proof of Corollary~\ref{corollary:svrg2}, we have $6T_0 \leq 7 (F(x_0)-F^\star)$.
Then, after taking the expectation with respect to the output of the first stage,
\begin{displaymath}
\begin{split}
   \E\left[F(\hat{x}_k)-F^\star\right] & \leq \Gamma_k \left( 8 \E[F(x_0)-F^\star] + {18 \rho_Q \tilde{\sigma}^2} \sum_{t=1}^k \frac{\eta_t \delta_t}{\Gamma_t}\right)   \\
    & \leq \Gamma_k \left( {192 \rho_Q \eta \tilde{\sigma}^2} + {18 \rho_Q\tilde{\sigma}^2} \sum_{t=1}^k \frac{\eta_t \delta_t}{\Gamma_t}     \right).
\end{split}
\end{displaymath}
Denote now by $k_0$ the largest index such that $\frac{2}{\mu(k_0+2)} \geq {\eta} $ and thus $k_0 = \lceil 2/(\mu {\eta}) - 2 \rceil$.
Then, according to Lemma~\ref{lemma:step}, for $k \geq k_0$,
\begin{displaymath}
\begin{split}
    \E\left[F(\hat{x}_k)-F^\star\right] 
   & \leq \Gamma_k \left( {192 \rho_Q \eta\tilde{\sigma}^2} + {18 \rho_Q {\eta} \tilde{\sigma}^2 }\sum_{t=1}^{k_0-1} \frac{\delta_t}{\Gamma_t} + {18 \rho_Q \tilde{\sigma}^2} \sum_{t=k_0}^{k} \frac{2\delta_t}{\mu \Gamma_t (t+2)}\right) \\
    & \leq \frac{k_0(k_0+1)}{(k+1)(k+2)} \left( \Gamma_{k_0-1}{192 \rho_Q {\eta} \tilde{\sigma}^2} + {18 {\eta} \rho_Q \tilde{\sigma}^2 } \Gamma_{k_0-1}\sum_{t=1}^{k_0-1} \frac{\delta_t}{\Gamma_t}\right) + {36\rho_Q \tilde{\sigma}^2} \sum_{t=k_0}^{k} \frac{\delta_t \Gamma_k}{\mu \Gamma_t (t+2)} \\
    & \leq \frac{k_0(k_0+1)}{(k+1)(k+2)} {192 {\eta} \rho_Q \tilde{\sigma}^2} + {36 \rho_Q \tilde{\sigma}^2} \sum_{t=k_0}^{k} \frac{(t+1)(t+2)}{\mu (k+1)(k+2) (t+2)^2} \\
    & \leq \frac{k_0{\eta}}{k+2} {192 \rho_Q \tilde{\sigma}^2} + \frac{36 \rho_Q \tilde{\sigma}^2}{\mu (k+2)} = O\left( \frac{\rho_Q \tilde{\sigma}^2}{\mu k}\right),\\
\end{split}
\end{displaymath}
which gives the desired complexity.

\end{proof}

\subsection{Proof of Theorem~\ref{thm:acc_sgd}}
\begin{proof}
First, the minimizer $v_k$ of the quadratic surrogate $d_k$ may be written as
\begin{displaymath}
\begin{split}
    v_k & = \frac{(1-\delta_k)\gamma_\kmone}{\gamma_k} v_\kmone + \frac{\mu \delta_k}{\gamma_k} y_\kmone - \frac{\delta_k}{\gamma_k} \tilde{g}_k \\
        & = y_\kmone + \frac{(1-\delta_k)\gamma_\kmone}{\gamma_k}(v_\kmone - y_\kmone) - \frac{\delta_k}{\gamma_k} \tilde{g}_k.
\end{split}
\end{displaymath}
Then, we characterize the quantity $d_k^\star$:
    \begin{displaymath}
    \begin{split}
       d_k^\star  & = d_k(y_\kmone) - \frac{\gamma_k}{2}\|v_k - y_\kmone\|^2 \\ 
                  & = (1-\delta_k)d_\kmone(y_\kmone) + \delta_k l_k(y_\kmone) - \frac{\gamma_k}{2}\|v_k - y_\kmone\|^2 \\ 
                  & = (1-\delta_k)\left(d_\kmone^\star + \frac{\gamma_\kmone}{2}\|y_\kmone- v_\kmone\|^2\right) + \delta_k l_k(y_\kmone) - \frac{\gamma_k}{2}\|v_k - y_\kmone\|^2 \\ 
                  & = (1-\delta_k)d_\kmone^\star  + \left(\frac{\gamma_\kmone(1-\delta_k)(\gamma_k - (1-\delta_k)\gamma_\kmone)}{2\gamma_k} \right)\|y_\kmone- v_\kmone\|^2 + \delta_k l_k(y_\kmone)  \\ 
                  & ~~~~~~~~~~~  - \frac{\delta_k^2}{2\gamma_k}\|\tilde{g}_k\|^2+ \frac{\delta_k (1-\delta_k)\gamma_\kmone}{\gamma_k} \tilde{g}_k^\top (v_\kmone - y_\kmone) \\
                  & \geq  (1-\delta_k)d_\kmone^\star   + \delta_k l_k(y_\kmone) - \frac{\delta_k^2}{2\gamma_k}\|\tilde{g}_k\|^2 
                   + \frac{\delta_k (1-\delta_k)\gamma_\kmone}{\gamma_k} \tilde{g}_k^\top (v_\kmone - y_\kmone).
    \end{split}
    \end{displaymath}
    Assuming by induction that $\E[d_\kmone^\star] \geq \E[F(x_\kmone)] - \xi_\kmone$ for some $\xi_\kmone \geq 0$, we have after taking expectation
    \begin{equation*}
       \E[d_k^\star]  \geq  (1-\delta_k)(\E[F(x_\kmone)] - \xi_\kmone) + \delta_k \E[l_k(y_\kmone)] - \frac{\delta_k^2}{2\gamma_k}\E\|\tilde{g}_k\|^2 
                   + \frac{\delta_k (1-\delta_k)\gamma_\kmone}{\gamma_k} \E[\tilde{g}_k^\top (v_\kmone - y_\kmone)].
    \end{equation*}
    Then, note that $\E[F(x_\kmone)] \geq \E[l_k(x_\kmone)] \geq \E[l_k(y_\kmone)] + \E[\tilde{g}_k^\top(x_\kmone-y_\kmone)]$, and 
    \begin{equation*}
       \E[d_k^\star] \geq  \E[l_k(y_\kmone)] - (1-\delta_k)\xi_\kmone - \frac{\delta_k^2}{2\gamma_k}\E\|\tilde{g}_k\|^2 
                   + (1-\delta_k)\E\left[\tilde{g}_k^\top \left(\frac{\delta_k\gamma_\kmone}{\gamma_k} (v_\kmone - y_\kmone) + (x_\kmone-y_\kmone)\right)\right].
    \end{equation*}
    By Lemma~\ref{lemma:acc}, we can show that the last term is equal to zero, and we are left with
    \begin{equation*}
       \E[d_k^\star] \geq  \E[l_k(y_\kmone)] - (1-\delta_k)\xi_\kmone -
       \frac{\delta_k^2}{2\gamma_k}\E\|\tilde{g}_k\|^2.
    \end{equation*}
    We may then use Lemma~\ref{lemma:key_acc}, which gives us
    \begin{equation*}
    \begin{split}
       \E[d_k^\star] & \geq  \E[F(x_k)] - (1-\delta_k)\xi_\kmone - \eta_k \sigma_k^2
       + \left( \eta_k - \frac{L \eta_k^2}{2} - \frac{\delta_k^2}{2\gamma_k}\right)\E\|\tilde{g}_k\|^2 \\
         & \geq \E[F(x_k)] - \xi_k~~~~\text{with}~~~~ \xi_k = (1-\delta_k)\xi_\kmone + \eta_k \sigma_k^2,
    \end{split}
    \end{equation*}
    where we used the fact that $\eta_k \leq 1/L$ and $\delta_k=\sqrt{\gamma_k \eta_k}$.

    It remains to choose $d_0^\star = F(x_0)$ and $\xi_0=0$ to initialize the induction at $k=0$ and we conclude that 
    \begin{equation*}
        \E\left[F(x_k) - F^\star + \frac{\gamma_k}{2}\|v_k-x^\star\|^2\right] \leq \E[d_k(x^\star)-F^\star] + \xi_k \leq \Gamma_k (d_0(x^\star) - F^\star) + \xi_k,  %\label{eq:acc_aux1}
    \end{equation*}
    which gives us the desired result when noticing that $\xi_k = \Gamma_k \sum_{t=1}^k \frac{\eta_t \sigma_t^2}{\Gamma_t}$.
\end{proof}

\subsection{Proof of Lemma~\ref{lemma:acc}}
\begin{proof}
Let us assume that the relation $y_\kmone = \theta_\kmone x_\kmone + (1-\theta_\kmone) v_\kmone$ holds and let us show that it also holds for $y_{k}$.
   Since the estimate sequences $d_k$ are quadratic functions, we have
   \begin{displaymath}
   \begin{split}
       v_k & = (1-\delta_k)\frac{\gamma_\kmone}{\gamma_k} v_{\kmone} + \frac{\mu \delta_k}{\gamma_k} y_\kmone - \frac{\delta_k}{\gamma_k}(g_k + \psi'(x_k)) \\
           & = (1-\delta_k)\frac{\gamma_\kmone}{\gamma_k} v_{\kmone} + \frac{\mu \delta_k}{\gamma_k} y_\kmone - \frac{\delta_k}{\gamma_k\eta_k}(y_\kmone - x_k) \\
           & = (1-\delta_k)\frac{\gamma_\kmone}{\gamma_k(1-\theta_\kmone)} \left(y_\kmone - \theta_\kmone x_\kmone  \right) + \frac{\mu \delta_k}{\gamma_k} y_\kmone - \frac{\delta_k}{\gamma_k\eta_k}(y_\kmone - x_k) \\
           & = (1-\delta_k)\frac{\gamma_\kmone}{\gamma_k(1-\theta_\kmone)} \left(y_\kmone - \theta_\kmone x_\kmone  \right) + \frac{\mu \delta_k}{\gamma_k} y_\kmone - \frac{1}{\delta_k}(y_\kmone - x_k) \\
           & = \left(\frac{(1-\delta_k)\gamma_\kmone}{\gamma_k(1-\theta_\kmone)}  + \frac{\mu \delta_k}{\gamma_k} - \frac{1}{\delta_k} \right)y_\kmone   -  \frac{(1-\delta_k)\gamma_\kmone\theta_\kmone}{\gamma_k(1-\theta_\kmone)} x_\kmone   + \frac{1}{\delta_k}x_k \\
           & = \left( 1 + \frac{(1-\delta_k)\gamma_\kmone \theta_\kmone}{\gamma_k(1-\theta_\kmone)} - \frac{1}{\delta_k} \right)y_\kmone   -  \frac{(1-\delta_k)\gamma_\kmone\theta_\kmone}{\gamma_k(1-\theta_\kmone)} x_\kmone   + \frac{1}{\delta_k}x_k.
   \end{split}
   \end{displaymath}
   Then note that $1-\theta_\kmone = \frac{\delta_k \gamma_\kmone}{\gamma_\kmone + \delta_k \mu}$ and thus, $\frac{\gamma_\kmone \theta_\kmone}{\gamma_k (1-\theta_\kmone)} = \frac{1}{\delta_k}$, and
   \begin{displaymath}
   \begin{split}
       v_k & = x_\kmone + 
    \frac{1}{\delta_k}(x_k-x_\kmone). 
   \end{split}
   \end{displaymath}
   Then, we note that $x_k - x_\kmone = \frac{\delta_k}{1-\delta_k}(v_k  - x_k)$ and we are left with
   \begin{displaymath}
       y_k = x_k + \beta_k(x_k-x_\kmone) = \frac{\beta_k \delta_k}{1-\delta_k} v_k  +   \left(  1-\frac{\beta_k \delta_k}{1-\delta_k}\right) x_k.
   \end{displaymath}
   Then, it is easy to show that
   \begin{displaymath}
   \beta_k = \frac{(1-\delta_k)\delta_{k+1} \gamma_k}{\delta_k( \gamma_{k+1} + \delta_{k+1}\gamma_k)} = \frac{(1-\delta_k)\delta_{k+1} \gamma_k}{\delta_k( \gamma_{k} + \delta_{k+1}\mu)} = \frac{(1-\delta_k)(1-\theta_k)}{\delta_k} ,
   \end{displaymath}
   which allows us to conclude that  $y_k = \theta_k x_k + (1-\theta_k) v_k$ since the relation holds trivially for $k=0$.
\end{proof}

\subsection{Proof of Lemma~\ref{lemma:key_acc}}
\begin{proof}
\begin{displaymath}
\begin{split}
   \E[F(x_k)] & = \E[f(x_k) + \psi(x_k)] \\
          & \leq \E\left[ f(y_\kmone) + \nabla f(y_\kmone)^\top (x_k - y_\kmone) + \frac{L}{2}\|x_k-y_\kmone\|^2 + \psi(x_k)\right] \\
          & = \E\left[f(y_\kmone) + g_k^\top (x_k - y_\kmone) + \frac{L}{2}\|x_k-y_\kmone\|^2 + \psi(x_k)\right]  +  \E\left[ (\nabla f(y_\kmone)-g_k)^\top (x_k - y_\kmone) \right] \\
          & = \E\left[f(y_\kmone) + g_k^\top (x_k - y_\kmone) + \frac{L}{2}\|x_k-y_\kmone\|^2 + \psi(x_k)\right]  +  \E\left[ (\nabla f(y_\kmone)-g_k)^\top x_k \right] \\
          & = \E\left[f(y_\kmone) + g_k^\top (x_k - y_\kmone) + \frac{L}{2}\|x_k-y_\kmone\|^2 + \psi(x_k)\right]  +  \E\left[ (\nabla f(y_\kmone)-g_k)^\top (x_k - w_\kmone) \right] \\
          & \leq \E\left[f(y_\kmone) + g_k^\top (x_k - y_\kmone) + \frac{L}{2}\|x_k-y_\kmone\|^2 + \psi(x_k)\right]  +  \E\left[ \|\nabla f(y_\kmone)-g_k \|\| x_k - w_\kmone \| \right] \\
          & \leq \E\left[f(y_\kmone) + g_k^\top (x_k - y_\kmone) + \frac{L}{2}\|x_k-y_\kmone\|^2 + \psi(x_k)\right]  +  \E\left[ \eta_k\|\nabla f(y_\kmone)-g_k \|^2\right] \\
          & = \E\left[l_k(y_\kmone) + \tilde{g}_k^\top (x_k - y_\kmone) + \frac{L}{2}\|x_k-y_\kmone\|^2\right]  +  \eta_k\sigma_k^2, \\
          & \leq \E\left[l_k(y_\kmone)\right] + \left(\frac{L\eta_k^2}{2} - \eta_k\right)\E\left[\|\tilde{g}_k\|^2\right]  +  \eta_k\sigma_k^2, \\
\end{split}
\end{displaymath}
where $w_\kmone = \Prox_{\eta_k\psi}[y_\kmone - \eta_k \nabla f(y_\kmone)]$. The first inequality is due to the $L$-smoothness of $f$ (Lemma~\ref{lemma:upper}); then, the next three relations exploit the fact that $\E[(\nabla f(y_\kmone)-g_k)^\top z = 0$ for all $z$ that is deterministic (which is the case for $y_\kmone$ and $w_\kmone$); the second inequality uses the non-expansiveness of the proximal operator. 
Then, we use the fact that $x_k = y_\kmone - \eta_k \tilde{g}_k$.
% \begin{displaymath}
% \begin{split}
%    \E[F(x_k)] & \leq \E\left[f(y_\kmone) + g_k^\top (x_k - y_\kmone) + \frac{L}{2}\|x_k-y_\kmone\|^2 + \psi(x_k)\right]  +  \eta_k \sigma_k^2, \\
% \end{split}
% \end{displaymath}
\end{proof}

% \subsection{Proof of Lemma~\ref{lemma:key_acc2}}
% \begin{proof}
% We start from an inequality obtained in the proof of Lemma~\ref{lemma:key_acc}.
% \begin{displaymath}
% \begin{split}
%    \E[F(x_k)] & \leq \E\left[f(y_\kmone) + g_k^\top (x_k - y_\kmone) + \frac{L}{2}\|x_k-y_\kmone\|^2 + \psi(x_k)\right] + \eta_k \sigma_k^2 \\
%               & \leq \E\left[f(y_\kmone) + g_k^\top (x_k - y_\kmone) + \frac{1}{2\eta_k}\|x_k-y_\kmone\|^2 + \psi(x_k)\right] + \eta_k \sigma_k^2.
% \end{split}
% \end{displaymath}
% Then, note that $x_k$ minimizes the strongly convex function $x \mapsto g_k^\top (x - y_\kmone) + \frac{1}{2\eta_k}\|x-y_\kmone\|^2 + \psi(x)$ such that, from Lemma~\ref{lemma:second},
% \begin{displaymath}
% \begin{split}
%    \E[F(x_k)] & \leq  \E\left[f(y_\kmone) + g_k^\top (x^\star - y_\kmone) + \frac{1}{2\eta_k}\|x^\star-y_\kmone\|^2 + \psi(x^\star) - \frac{1}{2\eta_k}\|x^\star - x_k\|^2 \right] + \eta_k \sigma_k^2 \\
%               & =  \E\left[f(y_\kmone) + \nabla f(y_\kmone)^\top (x^\star - y_\kmone) + \frac{1}{2\eta_k}\|x^\star-y_\kmone\|^2 + \psi(x^\star) - \frac{1}{2\eta_k}\|x^\star - x_k\|^2 \right] + \eta_k \sigma_k^2 \\
%               & \leq \E\left[f(x^\star)  + \frac{1}{2\eta_k}\|x^\star-y_\kmone\|^2 + \psi(x^\star) - \frac{1}{2\eta_k}\|x^\star - x_k\|^2 \right] + \eta_k \sigma_k^2 \\
%               & = F^\star + \E\left[\frac{1}{2\eta_k}\|x^\star-y_\kmone\|^2  - \frac{1}{2\eta_k}\|x^\star - x_k\|^2 \right] + \eta_k \sigma_k^2.
% \end{split}
% \end{displaymath}
% \end{proof}
% 

\subsection{Proof of Corollary~\ref{corollary:acc_sgd2}}
\begin{proof}
The proof is similar to that of Corollary~\ref{corollary:sgd} for unaccelerated SGD.
The first stage with constant step-size requires $O\left( \sqrt{\frac{L}{\mu}} \log\left(\frac{F(x_0)- F^\star}{\varepsilon}\right)\right)$ iterations. Then, we restart the optimization
procedure, and assume that $\E\left[F(x_0)-F^\star + \frac{\mu}{2}\|x^\star-x_0\|^2\right] \leq \frac{2\sigma^2}{\sqrt{\mu L}}$.
With the choice of parameters, we have $\gamma_k = \mu$ and $\delta_k = \sqrt{\gamma_k \eta_k} = \min\left( \sqrt{ \frac{\mu}{L} }, \frac{2}{k+2} \right)$. We may then apply Theorem~\ref{thm:acc_sgd} where the value of $\Gamma_k$ is given by Lemma~\ref{lemma:step}. This yields for $k \geq k_0 = \left\lceil 2\sqrt{\frac{L}{\mu}} - 2 \right\rceil$,
\begin{equation*}
\begin{split}
   \E[F({x}_k)- F^\star]  & \leq \Gamma_k\left( \E\left[F(x_0)-F^\star + \frac{\mu}{2}\|x_0-x^\star\|^2\right] + \sigma^2 \sum_{t=1}^k \frac{\eta_t}{\Gamma_t} \right) \\
   & \leq  \Gamma_k \left( \frac{2 \sigma^2}{\sqrt{\mu L}} + \frac{\sigma^2}{L} \sum_{t=1}^{k_0-1} \frac{1}{\Gamma_t} + \sigma^2 \sum_{t=k_0}^{k} \frac{4}{\Gamma_t \mu(t+2)^2}\right) \\
   & =  \frac{k_0(k_0+1)}{(k+1)(k+2)} \left( \Gamma_{k_0-1}\frac{2 \sigma^2}{\sqrt{ \mu L}} + \frac{\sigma^2}{L} \Gamma_{k_0-1}\sum_{t=1}^{k_0-1} \frac{1}{\Gamma_t}\right) + \sigma^2 \sum_{t=k_0}^{k} \frac{4\Gamma_k}{\Gamma_t \mu(t+2)^2} \\
   & =  \frac{k_0(k_0+1)}{(k+1)(k+2)} \left( \Gamma_{k_0-1}\frac{2 \sigma^2}{\sqrt{ \mu L}} + (1-\Gamma_{k_0 -1 })\frac{\sigma^2}{\sqrt{\mu L}}\right) + \sigma^2 \sum_{t=k_0}^{k} \frac{4 \Gamma_k}{\Gamma_t \mu(t+2)^2} \\
    & \leq  \frac{k_0(k_0+1)}{(k+1)(k+2)} \frac{2 \sigma^2}{\sqrt{\mu L}} +  \sigma^2 \frac{1}{(k+1)(k+2)} \left(\sum_{t=k_0+1}^{k} \frac{4 (t+1)(t+2)}{\mu(t+2)^2}\right) \\
    & \leq  \frac{k_0}{(k+1)(k+2)} \frac{4 \sigma^2}{\mu} +   \frac{4\sigma^2}{\mu(k+2)}  \leq  \frac{8\sigma^2}{\mu(k+2)},
   \end{split}
\end{equation*}
where we use Lemmas~\ref{lemma:step} and~\ref{lemma:simple}. 
This leads to the desired iteration complexity.
\end{proof}

\subsection{Proof of Proposition~\ref{prop:nonu2}}
\begin{proof}
\begin{equation*}
\begin{split}
   \sigma_k^2 & =  \E\left\|\frac{1}{q_{i_k}n} \left(\tildenabla f_{i_k}(y_{\kmone}) - \tildenabla f_{i_k}(\tilde{x}_\kmone) \right) + \tildenabla f(\tilde{x}_\kmone) - \nabla f(y_\kmone)\right\|^2 \\
              & =  \E\left\|\frac{1}{q_{i_k}n} \left(\nabla f_{i_k}(y_{\kmone}) + \zeta_k - \zeta'_k - \nabla f_{i_k}(\tilde{x}_\kmone) \right) + \nabla f(\tilde{x}_\kmone) + \bar{\zeta}_\kmone - \nabla f(y_\kmone) \right\|^2, \\
              & \leq  \E\left\|\frac{1}{q_{i_k}n} \left(\nabla f_{i_k}(y_{\kmone}) - \nabla f_{i_k}(\tilde{x}_\kmone) \right) + \nabla f(\tilde{x}_\kmone) + \bar{\zeta}_\kmone - \nabla f(y_\kmone) \right\|^2 + {2 \rho_Q \tilde{\sigma}^2}, \\
\end{split}
\end{equation*}
where $\zeta_k$ and $\zeta'_k$ are perturbations drawn at iteration $k$, and $\bar{\zeta}_\kmone$ was drawn last time $\tilde{x}_\kmone$ was updated.
Then, by noticing that for any deterministic quantity $Y$ and random variable $X$, we have $\E[\|X-\E[X] - Y\|^2] \leq \E[\|X\|^2] + \|Y\|^2$, taking expectation with respect to the index $i_k \sim Q$ and conditioning on~$\Fcal_\kmone$, we have
\begin{equation}
   \begin{split}
      \sigma_k^2 & \leq  \E\left\|\frac{1}{q_{i_k}n} \left(\nabla f_{i_k}(y_{\kmone}) - \nabla f_{i_k}(\tilde{x}_\kmone) \right) \right\|^2 + \E[\|\bar{\zeta}_\kmone\|^2]  + {2 \rho_Q \tilde{\sigma}^2} \\
      & \leq  \frac{1}{n}\sum_{i=1}^n \frac{1}{q_i n}\E\left\|\nabla f_{i}(y_{\kmone}) - \nabla f_{i}(\tilde{x}_\kmone)\right\|^2  + {3\rho_Q \tilde{\sigma}^2} \\
      & \leq  \frac{1}{n}\sum_{i=1}^n \frac{2 L_i}{q_i n}\E\left[ f_{i}(\tilde{x}_\kmone) - f_{i}(y_{\kmone}) -  \nabla f_{i}(y_{\kmone})^\top (\tilde{x}_\kmone-y_\kmone)\right]  + {3 \rho_Q \tilde{\sigma}^2} \\
      & \leq  \frac{1}{n}\sum_{i=1}^n {2 L_Q}\E\left[ f_{i}(\tilde{x}_\kmone) - f_{i}(y_{\kmone}) -  \nabla f_{i}(y_{\kmone})^\top (\tilde{x}_\kmone-y_\kmone)\right]  + {3 \rho_Q \tilde{\sigma}^2} \\
      & =  {2 L_Q}\E\left[ f(\tilde{x}_\kmone) - f(y_{\kmone}) -  \nabla f(y_{\kmone})^\top (\tilde{x}_\kmone-y_\kmone)\right]  + {3 \rho_Q \tilde{\sigma}^2} \\
      & =  {2 L_Q}\E\left[ f(\tilde{x}_\kmone) - f(y_{\kmone}) -  g_k^\top (\tilde{x}_\kmone-y_\kmone)\right]  + {3 \rho_Q \tilde{\sigma}^2},
   \end{split} \label{eq:variance_svrg}
\end{equation}
where the second inequality uses the  upper-bound $\E[\|\bar{\zeta}\|^2] = \frac{\sigma^2}{n} \leq {\rho_Q \sigma^2}$, and the third one uses Theorem 2.1.5 in~\cite{nesterov}.
\end{proof}

\subsection{Proof of Lemma~\ref{lemma:key_acc_svrg}}
\begin{proof}
We can show that Lemma~\ref{lemma:key_acc} still holds and thus, 
   \begin{displaymath}
   \begin{split}
   \E[F(x_k)] & \leq 
   \E\left[l_k(y_\kmone) \right] + \left(\frac{L\eta_k^2}{2} - \eta_k\right)\E\left[\|\tilde{g}_k\|^2\right]  +  \eta_k\sigma_k^2. \\
              & \leq  \E\left[l_k(y_\kmone) + a_k f(\tilde{x}_\kmone) - a_k f(y_\kmone) + a_k g_k^\top (y_\kmone - \tilde{x}_\kmone)\right] \\
              & \qquad\qquad\qquad\qquad + \E\left[\left(\frac{L\eta_k^2}{2} - \eta_k\right)\|\tilde{g}_k\|^2\right] + {3\rho_Q \eta_k\tilde{\sigma}^2}, \\
   \end{split}
   \end{displaymath}
   Note also that
   \begin{displaymath}
   \begin{split}
      l_k(y_\kmone) + f(\tilde{x}_\kmone) - f(y_\kmone) & = \psi(x_k) + \psi'(x_k)^\top(y_\kmone - x_k) + f(\tilde{x}_\kmone) \\ 
                                                        & \leq \psi(\tilde{x}_\kmone) - \psi'(x_k)^\top(\tilde{x}_\kmone - x_k) + \psi'(x_k)^\top(y_\kmone - x_k) + f(\tilde{x}_\kmone)  \\
                                                        & = F(\tilde{x}_\kmone) + \psi'(x_k)^\top(y_\kmone-\tilde{x}_\kmone).
   \end{split}
   \end{displaymath}
   Therefore, by noting that $l_k(y_\kmone) + a_k f(\tilde{x}_\kmone) - a_k f(y_\kmone) \leq (1-a_k)l_k(y_\kmone) + a_k F(\tilde{x}_\kmone) + a_k \psi'(x_k)^\top(y_\kmone-\tilde{x}_\kmone)$,  we obtain the desired result.
\end{proof}
\subsection{Proof of Theorem~\ref{thrm:acc_svrg}}
\begin{proof}
Following similar steps as in the proof of Theorem~\ref{thm:acc_sgd}, we have
    \begin{displaymath}
        d_k^\star \geq  (1-\delta_k)d_\kmone^\star   + \delta_k l_k(y_\kmone) - \frac{\delta_k^2}{2\gamma_k}\|\tilde{g}_k\|^2 
                    + \frac{\delta_k (1-\delta_k)\gamma_\kmone}{\gamma_k} \tilde{g}_k^\top (v_\kmone - y_\kmone).
    \end{displaymath}
   Assume now by induction that $\E[d_\kmone^\star] \geq \E[F(\tilde{x}_\kmone)] - \xi_\kmone$ for some $\xi_\kmone \geq 0$ and
    note that $\delta_k \leq \frac{1-a_k}{n}$ since $a_k = 2L_Q\eta_k \leq \frac{2}{3}$ and $\delta_k = \sqrt{\frac{5\eta_k\gamma_k}{3n}} \leq \frac{1}{3n} \leq \frac{1-a_k}{n}$. Then,
     \begin{displaymath}
        \begin{split}
           \E[d_k^\star] & \geq (1-\delta_k) (\E[F(\tilde{x}_\kmone)] - \xi_\kmone) + \delta_k\E[l_k(y_\kmone)] - \frac{\delta_k^2}{2\gamma_k}\E[\|\tilde{g}_k\|^2] 
           + \E\left[\tilde{g}_k^\top \left(\frac{\delta_k (1-\delta_k)\gamma_\kmone}{\gamma_k} (v_\kmone - y_\kmone)\right)\right] \\
           & \geq \left(1-\frac{1-a_k}{n}\right) \E[F(\tilde{x}_\kmone)] + \left(\frac{1-a_k}{n}-\delta_k\right) \E[F(\tilde{x}_\kmone)]  + \delta_k\E[l_k(y_\kmone)] - \frac{\delta_k^2}{2\gamma_k}\E[\|\tilde{g}_k\|^2] \\
           & \qquad \qquad\qquad\qquad
           + \E\left[\tilde{g}_k^\top \left(\frac{\delta_k (1-\delta_k)\gamma_\kmone}{\gamma_k} (v_\kmone - y_\kmone)\right)\right] - (1-\delta_k)\xi_\kmone.
        \end{split}
     \end{displaymath}
     Note that
     \begin{displaymath}
        \E[F(\tilde{x}_\kmone)] \geq \E[l_k(\tilde{x}_\kmone)] \geq \E[l_k(y_\kmone)] + \E[\tilde{g}_k^\top(\tilde{x}_\kmone - y_\kmone)].
     \end{displaymath}
     Then,
     \begin{multline*}
        \E[d_k^\star] \geq \left(1-\frac{1-a_k}{n}\right)\E[F(\tilde{x}_\kmone)] + \frac{1-a_k}{n}\E[l_k(y_\kmone)] - \frac{\delta_k^2}{2\gamma_k}\E[\|\tilde{g}_k\|^2] \\ + \E\left[\tilde{g}_k^\top \left(\frac{\delta_k (1-\delta_k)\gamma_\kmone}{\gamma_k} (v_\kmone - y_\kmone) + \left(\frac{1-a_k}{n}-\delta_k\right)(\tilde{x}_\kmone - y_\kmone) \right)\right] - (1-\delta_k)\xi_\kmone.
     \end{multline*}
     We may now use Lemma~\ref{lemma:key_acc_svrg}, which gives us
     \begin{multline}
        \E[d_k^\star] \geq \left(1-\frac{1}{n}\right) \E[F(\tilde{x}_\kmone)] + \frac{1}{n}\E[F(x_k)] + \left(\frac{1}{n}\left(\eta_k - \frac{L \eta_k^2}{2}\right)- \frac{\delta_k^2}{2\gamma_k}\right)\E[\|\tilde{g}_k\|^2] \\ + \E\left[\tilde{g}_k^\top \left(\frac{\delta_k (1-\delta_k)\gamma_\kmone}{\gamma_k} (v_\kmone - y_\kmone) + \left(\frac{1}{n}-\delta_k\right)(\tilde{x}_\kmone - y_\kmone)\right) \right] - \xi_k, \label{eq:aux_svrg_acc}
     \end{multline}
     with $\xi_k = (1-\delta_k)\xi_\kmone + \frac{3 \rho_Q \eta_k \tilde{\sigma}^2}{n}$.
     Then, since $\delta_k = \sqrt{\frac{5\eta_k\gamma_k}{3n}}$ and $\eta_k \leq \frac{1}{3L_Q} \leq \frac{1}{3L}$,
     \begin{displaymath}
        \frac{1}{n}\left(\eta_k - \frac{L \eta_k^2}{2}\right)- \frac{\delta_k^2}{2\gamma_k} \geq \frac{5\eta_k}{6n}- \frac{\delta_k^2}{2\gamma_k} = 0,
     \end{displaymath}
     and the term in~(\ref{eq:aux_svrg_acc}) involving $\|\tilde{g}_k\|^2$ may disappear. Similarly, we have
     \begin{displaymath}
         \frac{\delta_k(1-\delta_k)\gamma_{\kmone}}{\delta_k(1-\delta_k)\gamma_{\kmone} + \gamma_k/n - \delta_k \gamma_k}  = \frac{\delta_k\gamma_k - \delta_k^2\mu}{\gamma_k/n - \delta_k^2\mu} = \frac{3n\delta_k^3/5 \eta_k - \delta_k^2\mu}{3 \delta_k^2/5\eta_k - \delta_k^2\mu}= \frac{3n - 5\mu\eta_k}{3 - 5\mu\eta_k}= \theta_k,
     \end{displaymath}
     and the term in~(\ref{eq:aux_svrg_acc}) that is linear in $\tilde{g}_k$ may disappear as well.
     Then, we are left with
     $\E[d_k^\star] \geq \E[F(\tilde{x}_k)] - \xi_k$. Initializing the induction requires choosing $\xi_0=0$ and $d_0^\star = F(x_0)$. Ultimately, we note that $\E[d_k(x^\star)-F^\star] \leq (1-\delta_k)\E[d_\kmone(x^\star)-F^\star]$ for all $k \geq 1$, and
     $$ \E\left[ F(\tilde{x}_k) - F^\star \!+\! \frac{\gamma_k}{2}\|x^\star-v_k\|^2\right] \leq \E[d_k(x^\star)-F^\star] + \xi_k \leq \Gamma_k\left(F(x_0) - F^\star \!+\! \frac{\gamma_0}{2}\|x^\star \!-\! x_0\|^2 \right) + \xi_k,$$
     and we obtain the desired result. 
\end{proof}

\subsection{Proof of Corollary~\ref{corollary:acc_svrg}}
\begin{proof}
The proof is similar to that of Corollary~\ref{corollary:acc_sgd2} for accelerated SGD.
The first stage with constant step-size $\eta$ requires $O\left( \left(n + \sqrt{\frac{nL_Q}{\mu}}\right) \log\left(\frac{F(x_0)- F^\star}{\varepsilon}\right)\right)$ iterations. 
Then, we restart the optimization
procedure, and assume that $\E\left[F(x_0)-F^\star\right] \leq B$ with $B = 3\rho_Q\tilde{\sigma}^2 \sqrt{\eta/\mu n}$.

With the choice of parameters, we have $\gamma_k = \mu$ and $\delta_k = \sqrt{\frac{5\mu \eta_k}{3n}} = \min\left( \sqrt{\frac{5\mu \eta}{3n}}, \frac{2}{k+2} \right)$. We may then apply Theorem~\ref{thrm:acc_svrg} where the value of $\Gamma_k$ is given by Lemma~\ref{lemma:step}. This yields for $k \geq k_0 = \left\lceil \sqrt{\frac{12 n}{5 \mu \eta}} - 2 \right\rceil$,
\begin{equation*}
\begin{split}
   \E[F({x}_k)- F^\star]  & \leq \Gamma_k\left( \E\left[F(x_0)-F^\star + \frac{\mu}{2}\|x_0-x^\star\|^2\right] + \frac{ 3 \rho_Q \tilde{\sigma}^2}{n} \sum_{t=1}^k \frac{\eta_t}{\Gamma_t} \right) \\
   & \leq  \Gamma_k \left( 2 B + \frac{3 \rho_Q \tilde{\sigma}^2 \eta }{n} \sum_{t=1}^{k_0-1} \frac{1}{\Gamma_t} +  \frac{3 \rho_Q \tilde{\sigma}^2}{n}\sum_{t=k_0}^{k} \frac{12 n}{5\Gamma_t \mu(t+2)^2}\right) \\
   & =  \frac{k_0(k_0+1)}{(k+1)(k+2)} \left( \Gamma_{k_0-1}2B + \frac{3 \rho_Q \tilde{\sigma}^2 \eta }{n} \Gamma_{k_0-1}\sum_{t=1}^{k_0-1} \frac{1}{\Gamma_t}\right) +  \frac{36 \rho_Q \tilde{\sigma}^2}{5\mu}\sum_{t=k_0}^{k} \frac{\Gamma_k}{\Gamma_t(t+2)^2} \\
   & =  \frac{k_0(k_0+1)}{(k+1)(k+2)} \left( \Gamma_{k_0-1}2B + (1-\Gamma_{k_0 -1 }) \frac{3 \rho_Q \tilde{\sigma}^2 \eta }{n\delta_{k_0}}\right) +  \frac{36 \rho_Q \tilde{\sigma}^2}{5\mu}\sum_{t=k_0}^{k} \frac{\Gamma_k}{\Gamma_t(t+2)^2} \\
    & \leq  \frac{2 k_0(k_0+1)B}{(k+1)(k+2)}  +   \frac{8 \rho_Q \tilde{\sigma}^2}{\mu(k+1)(k+2)} \left(\sum_{t=k_0+1}^{k} \frac{(t+1)(t+2)}{(t+2)^2}\right) \\
    & \leq  \frac{2 k_0 B}{k+2}  +   \frac{8 \rho_Q \tilde{\sigma}^2}{\mu(k+2)}, \\
   \end{split}
\end{equation*}
where we use Lemmas~\ref{lemma:step} and~\ref{lemma:simple}. Then, note that $k_0 B \leq 6 \rho_Q \tilde{\sigma}^2/\mu$ and we obtain  
the right  iteration complexity.
\end{proof}

%\includepdf[pages=-]{../main_anonymous.pdf}

\end{document}